\documentclass{article}
 \newcommand{\JS}[2]{D_{\mathrm{JS}}\!\left(#1\;\middle\|\;#2\right)}
 \usepackage{lmodern}
 \newcommand{\calY}{\mathcal{Y}} 
 \usepackage[main, final]{neurips_2026}
\setcitestyle{authoryear,round,semicolon}

\usepackage[utf8]{inputenc}
\usepackage[T1]{fontenc}
\usepackage{lmodern}
\usepackage{microtype}          

\usepackage{amsmath}
\usepackage{amssymb}
\usepackage{amsfonts}
\usepackage{amsthm}
\usepackage{mathtools}          
\usepackage{nicefrac}           
\usepackage{bm}                 

\usepackage{xcolor}
\usepackage{enumitem}           
\usepackage{booktabs}           
\usepackage{multirow}
\usepackage{makecell}
\usepackage{tablefootnote}
\usepackage{adjustbox}
\usepackage{subcaption}

\usepackage{graphicx}
\graphicspath{{./figures/}}

\usepackage{algorithm}
\usepackage{algorithmic}
 
\makeatletter
\renewcommand{\fnum@algorithm}{\textbf{Algorithm~\thealgorithm}}
\makeatother
 
\usepackage{url}
\usepackage{hyperref}
\hypersetup{
    colorlinks = true,
    linkcolor  = blue!60!black,
    citecolor  = blue!60!black,
    urlcolor   = blue!60!black,
}

\usepackage{thmtools}
\usepackage[most]{tcolorbox}
 
\definecolor{thmcream}{RGB}{253, 250, 240}
\definecolor{thmgold} {RGB}{189, 155,  96}
 
\tcbset{
  thmbox/.style={
    enhanced,
    breakable,
    colback    = thmcream,
    colframe   = thmgold,
    boxrule    = 0.6pt,
    arc        = 3pt,
    left=8pt, right=8pt,
    top=6pt,  bottom=6pt,
    fonttitle  = \bfseries,
    separator sign none,
    description font = \normalfont,
  }
}
 
\newtcbtheorem[number within=section]{theorem}  {Theorem}  {thmbox}{thm}
\newtcbtheorem[number within=section]{corollary}{Corollary}{thmbox}{cor}
\newtcbtheorem[number within=section]{definition}{Definition}{thmbox}{def}
 
\newtheorem{lemma}      {Lemma}     [section]
\newtheorem{assumption} {Assumption}[section]
\newtheorem{proposition}{Proposition}[section]
\newtheorem{remark}     {Remark}    [section]

\newcommand{\E}  {\mathbb{E}}        
\newcommand{\R}  {\mathbb{R}}        
\newcommand{\bP} {\mathbb{P}}        
 
\newcommand{\KL}[2]{D_{\mathrm{KL}}\!\left(#1\;\middle\|\;#2\right)}
\newcommand{\TV}[2]{D_{\mathrm{TV}}\!\left(#1\;\middle\|\;#2\right)}

\newcommand{\bx}{\mathbf{x}}    
\newcommand{\bu}{\mathbf{u}}    
\newcommand{\bv}{\mathbf{v}}    
\newcommand{\bw}{\mathbf{w}}    
\newcommand{\bg}{\mathbf{g}}    
\newcommand{\bz}{\mathbf{z}}    
 
\newcommand{\bV}{\mathbf{V}}    
 
\newcommand{\bpi}   {\boldsymbol{\pi}}      
\newcommand{\bmu}   {\boldsymbol{\mu}}      
\newcommand{\bSigma}{\boldsymbol{\Sigma}}   
\newcommand{\bsigma}{\boldsymbol{\sigma}}   
\newcommand{\balpha}{\boldsymbol{\alpha}}   
\newcommand{\bPhi}  {\boldsymbol{\Phi}}     
 
\newcommand{\calD}{\mathcal{D}}   
\newcommand{\calL}{\mathcal{L}}   
\newcommand{\calS}{\mathcal{S}}   
\newcommand{\calC}{\mathcal{C}}   

\newcommand{\effsize}{S_{\mathrm{eff}}^{(c)}}

\newcommand{\ripple}     {\textsc{Ripple}}
\newcommand{\rippleStrat}{\textsc{Ripple-Strat}}
\newcommand{\rippleUnif} {\textsc{Ripple-Unif}}
\newcommand{\rippleHDB}  {\textsc{Ripple-HDB}}
\newcommand{\hdbscan}    {\textsc{Hdbscan}}
 
\newcommand{\ifca}  {\textsc{Ifca}}
\newcommand{\fesem} {\textsc{FeSEM}}
\newcommand{\cfl}   {\textsc{Cfl}}
 
\newcommand{\fedavg}    {\textsc{FedAvg}}
\newcommand{\feddyn}    {\textsc{FedDyn}}
\newcommand{\scaffold}  {\textsc{Scaffold}}
\newcommand{\fedadam}   {\textsc{FedAdam}}
\newcommand{\fedyogi}   {\textsc{FedYogi}}
\newcommand{\fedadagrad}{\textsc{FedAdagrad}}
 
\newcommand{\perfedavg}{\textsc{Per-FedAvg}}
\newcommand{\pfedme}   {\textsc{pFedMe}}

\newcommand{\eg}{\emph{e.g.}}

\newcommand{\parheader}[1]{\medskip\noindent\textbf{#1.}\;}

\usepackage{cleveref}
\newcommand{\bO}{\mathbf{O}}    
\newcommand{\bI}{\mathbf{I}}    
\crefname{theorem}    {Theorem}    {Theorems}
\Crefname{theorem}    {Theorem}    {Theorems}
\crefname{corollary}  {Corollary}  {Corollaries}
\Crefname{corollary}  {Corollary}  {Corollaries}
\crefname{definition} {Definition} {Definitions}
\Crefname{definition} {Definition} {Definitions}
\crefname{lemma}      {Lemma}      {Lemmas}
\Crefname{lemma}      {Lemma}      {Lemmas}
\crefname{assumption} {Assumption} {Assumptions}
\Crefname{assumption} {Assumption} {Assumptions}
\crefname{proposition}{Proposition}{Propositions}
\Crefname{proposition}{Proposition}{Propositions}
\crefname{remark}     {Remark}     {Remarks}
\Crefname{remark}     {Remark}     {Remarks}
\crefname{algorithm}  {Algorithm}  {Algorithms}
\Crefname{algorithm}  {Algorithm}  {Algorithms}
\crefname{section}    {Section}    {Sections}
\Crefname{section}    {Section}    {Sections}
\crefname{appendix}   {Appendix}   {Appendices}
\Crefname{appendix}   {Appendix}   {Appendices}
\crefname{figure}     {Figure}     {Figures}
\Crefname{figure}     {Figure}     {Figures}
\crefname{table}      {Table}      {Tables}
\Crefname{table}      {Table}      {Tables}
\crefname{equation}   {Eq.}        {Eqs.}
\Crefname{equation}   {Eq.}        {Eqs.}
\title{\ripple{} in Still Water: Zero-Shot Clustering in Federated Learning\\
  with Wavelet Scattering Transform}

\author{
  Alessandro Licciardi \\
  Department of Mathematical Sciences\\
  Politecnico di Torino, Turin, Italy\\
  \texttt{alessandro.licciardi@polito.it}
}

\begin{document}

\maketitle

\begin{abstract}
Clustered Federated Learning (FL) partitions a client population into groups 
of similar local distributions and trains one specialized model per 
cluster, mitigating client drift that degrades single-model methods 
under non-IID data. Prior methods discover cluster structure inside the 
training loop through gradient similarity, loss evaluation, or 
EM-style updates, thus increasing communication overhead, exposing 
gradients to inversion attacks, and providing no mechanism to assign 
clients absent from training. We propose \ripple{}, a clustered FL 
framework in which cluster assignment is computed entirely offline from 
a spectral characterization of each client's local data: a 
variance-weighted principal-component prototype embedded via the Wavelet 
Scattering Transform and decoded by a Gaussian Mixture VAE trained 
server-side on synthetic client populations before federation begins. 
Per-round communication cost matches \fedavg{} exactly, and a client 
absent from training obtains a personalized model from a single forward 
pass, without gradient computation, model evaluation, or extra 
communication round. We prove that the gap between \ripple{}'s surrogate 
clustered objective and the oracle is bounded by a computable quantity 
decaying with client sample size and independent of federation duration; 
per-cluster convergence matches the minimax-optimal rate for non-convex 
smooth objectives. Across five benchmarks spanning controlled and realistic heterogeneity, 
\ripple{} consistently outperforms all baselines, with margins growing 
on the most realistic partitions.

\end{abstract}
\section{Introduction}
\label{sec:intro}

\begin{flushright}
\begin{minipage}{0.55\columnwidth}
\itshape\small
Ripple in still water,\\
when there is no pebble tossed,\\
nor wind to blow.\\[2pt]
\upshape\scriptsize\hfill --- R. Hunter \& J. Garcia (1970)
\end{minipage}
\end{flushright}
\vspace{4pt}

\noindent The latent clustered 
structure of a federation need not be stirred into being by training, 
but already ripples through the spectral geometry of its data.

Federated Learning (FL)  \citep{mcmahan2017communication,kairouz2021advances}
trains shared models without pooling private data, but its canonical
algorithm \fedavg{} \citep{mcmahan2017communication} silently assumes
that all clients sample from the same distribution.
Under the statistical heterogeneity that is ubiquitous in real
deployments \citep{zhao2018federated,hsieh2020non}, gradient updates from
clients with structurally distinct data point in conflicting directions, which \citet{karimireddy2020scaffold} called \textit{client drift},degrading every
client's accuracy at once.
Variance-reduction and adaptive
methods \citep{karimireddy2020scaffold,acar2021federated,reddi2021adaptive,li2020federated}
attenuate the symptom but do not address the cause: a single global model
cannot represent clients whose data lie in qualitatively different regimes.

Clustered FL addresses the cause directly, training one
specialized model per group of compatible
clients \citep{sattler2020clustered,ghosh2020efficient,long2023multi,fenoglio2025flux}.
Yet every prior method discovers cluster structure \emph{inside} the
training loop, e.g., through gradient cosine
similarity \citep{sattler2020clustered}, repeated loss
evaluation \citep{ghosh2020efficient}, or EM-style gradient
statistics \citep{long2023multi}. These methods accumulate communication overhead
with each round, exposing updates to gradient-inversion
attacks \citep{zhang2022fldetector}, and offering no pathway to assign a
client that was not part of training. Moreover, while prior methods such as \ifca{} and \fesem{} require the number of clusters to be specified a priori, \ripple{} needs only a loose upper bound, and \rippleHDB{} then automatically adapts the cluster count via density-based clustering on the posterior simplex.

We propose a different route, motivated by recent work showing that the
Wavelet Scattering Transform
(WST) \citep{mallat2012group,bruna2013invariant} can identify anomalous
clients in a federation from spectral signatures of their local
data \citep{licciardi2026waffle}, with no federated communication
whatsoever. The same spectral characterization extends naturally to
distinguishing \emph{semantically distinct} client populations: framing
the problem as probabilistic inference over a simplex-valued mixture
model yields a soft cluster-weight vector for each client from a single
forward pass, computed from local data statistics alone. This is the
central design principle of \ripple{}: cluster assignment is a property
of the data, not an artifact of training. The assignment is fixed before
the first communication round and never revised; per-round cost matches
\fedavg{} exactly; and a client absent from training receives a
personalized model by transmitting nothing but a small spectral
embedding of its own data.
\vspace{-1em}
\paragraph{Contributions.} We summarize the contributions as follows.
\begin{itemize}[leftmargin=*,itemsep=2pt,topsep=2pt]
    \item We reframe clustered FL routing as an offline statistic
    computed from local data, eliminating per-round communication 
    overhead and fixing the routing decision before the first federation 
    round.
    \item We prove a surrogate-fidelity certificate: the gap between
    \ripple{}'s clustered objective and the oracle is bounded by a
    computable quantity that decays with client sample size and is
    independent of the number of rounds.
    \item A client absent from training obtains a personalized model
    from a single WST embedding and one encoder forward pass, without any
    gradient computation or model evaluation required.
    \item \ripple{} outperforms all baselines across five benchmarks,
    with the largest margins on the most realistic
    partitions \citep{venkateswara2017deep,weyand2020google}; zero-shot
    assignment matches or exceeds other clustered FL methods.
\end{itemize}
\section{Related Work}
\label{sec:related}
\vspace{-1em}
\parheader{Federated learning under data heterogeneity}
Non-IID client data has been central to FL research since
\citet{zhao2018federated} quantified how label skew degrades
\fedavg{} \citep{mcmahan2017communication} convergence.
Variance-reduction methods address client drift directly:
\scaffold{} \citep{karimireddy2020scaffold} corrects local-SGD bias via
control variates, while \feddyn{} \citep{acar2021federated} adds a dynamic
regularization term penalizing divergence from the global model.
Adaptive-optimizer variants \citep{reddi2021adaptive} improve convergence
under heterogeneous gradients; FedProx \citep{li2020federated} constrains
local updates via a proximal term.
All these methods pursue a \emph{single} global model, which cannot fully
resolve clients whose data lie in fundamentally distinct semantic
regimes \citep{hsieh2020non}.
Personalized FL sidesteps this by training per-client models via
meta-learning (\perfedavg{} \citep{fallah2020personalized},
FedALA \citep{zhang2023fedala}), Moreau-envelope
regularization (\pfedme{} \citep{t2020personalized}), or shared
representations \citep{collins2021exploiting}; yet when clients share
significant within-group structure, per-client models discard exploitable
inter-client information that clustered training would leverage.

\parheader{Clustered federated learning}
Clustering has proven to be an effective strategy in FL for handling client heterogeneity and improving personalization \citep{huang2022collaboration, duan2021fedgroup, briggs2020federated, Caldarola_2021_CVPR, ye2023personalized}.
\citet{sattler2020clustered} group clients by cosine gradient similarity,
but the splitting threshold is sensitive and prone to over- or
under-fragmentation under strong non-IID.
\ifca{} \citep{ghosh2020efficient} alternates hard cluster assignment with
gradient descent, requiring each client to evaluate all $C$ models per
round.
\fesem{} \citep{long2023multi} uses EM-style soft assignments but still
iterates gradient updates from every candidate center.
Graph- and hierarchy-based methods \citep{ye2023personalized,
huang2022collaboration, briggs2020federated, duan2021fedgroup} discover
collaboration structure adaptively yet tie cluster inference to in-training
gradient statistics.
FedGWC \citep{licciardi2025interaction} and
FLUX \citep{fenoglio2025flux} refine aggregation with interaction-aware
weighting and data descriptors, respectively, the first by modeling interaction between clients' empirical losses, while the latter infer cluster structure from gradient feedback exchanged \emph{during} training.

All prior methods share three structural limitations.
\emph{(i) Hard or in-training assignment}: membership is either
binary \citep{sattler2020clustered, ghosh2020efficient} or requires ongoing
gradient communication \citep{long2023multi, licciardi2025interaction,
fenoglio2025flux}, coupling cluster discovery to the training loop.
\emph{(ii) No zero-shot capability}: a new client must participate in at
least one communication round or evaluate all $C$ models locally to obtain
an assignment.
\emph{(iii) Communication overhead above \fedavg{}}: discovering cluster
structure consumes extra bandwidth every round.
\ripple{} overcomes these limitations, as cluster weights $\bw_k \in \Delta^{C-1}$ are
inferred \emph{once}, before any federated round, from a frozen GMM-VAE
requiring no gradient exchange; per-round communication matches \fedavg{}
exactly; and a new client is assigned via a single forward pass of the
frozen encoder.


\section{Background}
\label{sec:background}

In cross-silo FL, a central server coordinates $K$ clients,
indexed $k \in [K] \coloneqq \{1,\dots,K\}$.
Each client $k$ holds a private local dataset
$\calD_k = \{(\bx_i, y_i)\}_{i=1}^{n_k}$ of $n_k$ labeled samples
drawn i.i.d. from a local distribution $\bP_k$ on
$\R^D\times\calY$, never shared
directly \citep{kairouz2021advances,bonawitz2019towards}.
The pre-federation pipeline operates on features $\bx_i$ alone and
requires no labels, which enter only through the loss
$\ell(\theta;\bx_i,y_i)$ during federated training.
Let $p_k = n_k / \sum_{j=1}^K n_j$ denote the normalized data weight
of client $k$. The canonical algorithm,
\fedavg{} \citep{mcmahan2017communication}, minimizes a weighted
empirical risk over a single global parameter vector
$\theta \in \R^{|\theta|}$:
\vspace{-.5em}
\begin{equation}
    \min_{\theta}\;\calL(\theta)
    \;:=\;
    \sum_{k=1}^{K} p_k\, F_k(\theta),
    \qquad
    F_k(\theta)
    \;=\; \frac{1}{n_k}\sum_{i=1}^{n_k} \ell(\theta;\bx_i,y_i),
    \label{eq:fedavg_obj}
\end{equation}
where $\ell:\R^{|\theta|}\times\R^D\times\calY\to\R$ is a
differentiable loss. Each round, the server broadcasts the current
model to a sampled client subset; each client runs $E$ steps of local
SGD from that initialization, and the server aggregates updates via a
data-weighted average. A full probabilistic derivation of this
objective and its clustered extension is given in
Appendix \ref{app:background:prob}.

\subsection{A Probabilistic Model of Client Heterogeneity}
\label{sec:bg:prob}

\vspace{-.5em}
We depart from the binary cluster-membership assumption and model each
client's data distribution as a \emph{convex mixture} of $C$
underlying semantic distributions.

\begin{definition}{Client Mixture Distribution}{client_mixture}
Let $\{\bP_1,\dots,\bP_C\}$ be $C$ semantic distributions on
$\R^D\times\calY$. Each client $k \in [K]$ observes its local data
i.i.d.\ from a latent mixture distribution:
\begin{equation}
    \bP_k^*
    \;=\; \sum_{c=1}^{C} \pi_{k,c}^*\,\bP_c,
    \label{eq:client_mixture}
\end{equation}
where the \emph{true cluster-weight vector}
$\bpi_k^* = (\pi_{k,1}^*,\dots,\pi_{k,C}^*)^\top \in
\Delta^{C-1} = \bigl\{\bpi \in \R_{\geq 0}^C :
\textstyle\sum_c \pi_c = 1\bigr\}$.
\end{definition}

\noindent
If $\{\bpi_k^*\}_{k=1}^K$ were known, the theoretically optimal
strategy minimizes the \emph{oracle clustered objective}:
\begin{equation}
    \calL_{\mathrm{true}}(\Theta)
    \;=\; \sum_{k=1}^{K} p_k \sum_{c=1}^{C} \pi_{k,c}^*\,
    F_k(\theta_c),
    \label{eq:oracle_obj}
\end{equation}
over cluster models $\Theta = \{\theta_1,\dots,\theta_C\}$.

\begin{remark}[Generality of the mixture model]
\label{rem:generality}
The oracle objective \eqref{eq:oracle_obj} subsumes classical FL as
special cases: \fedavg{} is recovered when $C=1$; fully personalized
FL corresponds to $C=K$ with $\pi_{k,k}^*=1$. The regime $1 < C < K$
with genuinely fractional $\bpi_k^*$ strictly generalizes prior
clustered FL \citep{ghosh2020efficient,sattler2020clustered,
long2023multi,licciardi2025interaction}, all of which restrict
$\bpi_k^* \in \{0,1\}^C$.
\end{remark}

\paragraph{The cluster assignment problem.}
In practice, $\bpi_k^*$ is entirely latent. Inferring client
similarities via gradient or parameter exchange incurs
$O(K^2|\theta|)$ communication overhead and creates a privacy
surface, since gradient updates can leak local data
information \citep{bonawitz2017practical,kairouz2021advances}. Our
goal is instead to estimate a \emph{surrogate weight vector}
$\bw_k \approx \bpi_k^*$ from each client's \emph{local,
pre-federation statistics alone}, and substitute it into the training
objective:
\begin{equation}
    \calL_{\mathrm{surrogate}}(\Theta)
    \;=\; \sum_{k=1}^{K} p_k \sum_{c=1}^{C} w_{k,c}\, F_k(\theta_c).
    \label{eq:surrogate_obj}
\end{equation}
As we show in Theorem \ref{thm:risk_bound}, controlling the $\ell_1$
discrepancy $\|\bw_k - \bpi_k^*\|_1$ uniformly over all clients is
both necessary and sufficient to bound the gap between
$\calL_{\mathrm{surrogate}}$ and $\calL_{\mathrm{true}}$.

\begin{remark}[Zero-shot cluster assignment]
\label{rem:routing}
We use \emph{cluster assignment} to refer to the map
$k \mapsto \bw_k \in \Delta^{C-1}$, generalizing hard partitioning
to soft probabilistic allocations. When computed before any
federation round using only local pre-federation statistics, without any gradient exchange or cluster-model evaluations.  We call this procedure
\emph{zero-shot cluster assignment}, borrowing the convention
from \citet{bengio2017deep}.
\end{remark}

\subsection{Wavelet Scattering Transform}
\label{sec:bg:wst}

The Wavelet Scattering Transform
(WST) \citep{mallat2012group,bruna2013invariant} builds a
parameter-free signal representation by cascading wavelet convolutions
with pointwise modulus nonlinearities, followed by lowpass averaging.
The resulting second-order representation $\bPhi(\bx)\in\R^P$
satisfies three properties that underpin our theoretical guarantees:
it is \emph{non-expansive} ($\|S_J\bx - S_J\bx'\|_2 \leq
\|\bx-\bx'\|_2$), \emph{translation-invariant} at
scale $2^J$, and \emph{Lipschitz-stable to diffeomorphisms} with
constant $C_2=2\|\nabla\psi\|_1$. Together these properties ensure that two
clients whose distributions differ only through sensor
miscalibrations, object translations, or mild domain shifts produce
representations $\bPhi_k$ at a distance proportional to the
magnitude of those perturbations, thus making $\bPhi$ a stable,
label-free fingerprint for pre-federation cluster assignment. The
full construction, formal proofs, and a sensitivity study are given
in Appendix \ref{app:background:wst}.

We refer to Appendix \ref{app:background} for a self-contained
probabilistic derivation of the federated and clustered objectives
from first principles, and for the full WST construction and
stability properties that underpin
Theorems \ref{thm:embed_stability}--\ref{thm:risk_bound}.
\section{The \ripple{} Algorithm}
\label{sec:method}

\ripple{} solves the cluster assignment problem zero-shot: before any
federation round, each client transmits a single deformation-stable
summary of its local distribution, and the server maps it to a soft
cluster-weight vector $\bw_k\in\Delta^{C-1}$. This supports clients
with genuinely mixed data modalities and generalizes to unseen clients
via a single encoder forward pass, following one architectural choice:
a Gaussian Mixture VAE trained server-side on synthetic client
populations before federation begins.

\subsection{Pre-Federation Distribution Characterization}
\label{sec:method:msdl}

\paragraph{Spectral prototype.}
Let $\hat{\bmu}_k$ and $\hat{\bSigma}_k$ be the empirical mean and
covariance of each client's data distribution, with eigendecomposition
$\hat{\bSigma}_k = \sum_j \hat{\lambda}_j \hat{\bv}_j\hat{\bv}_j^\top$.
We summarize each client by the variance-weighted principal-component
prototype:
\begin{equation}
    \hat{\bx}_k
    \;\coloneqq\;
    \hat{\bmu}_k
    \;+\;
    \sum_{j=1}^{M}
    \frac{\hat{\lambda}_j}{\sum_{l=1}^{M}\hat{\lambda}_l}\,\hat{\bv}_j,
    \label{eq:method:prototype}
\end{equation}
where $M$ is selected by the explained-variance criterion
$M^\star = \min\{M : \sum_{j\leq M}\hat\lambda_j /
\sum_{j\leq D}\hat\lambda_j \geq \rho\}$, where $\rho\in[0.90,0.99]$ is the explained variance threshold.

\paragraph{Wavelet Scattering embedding.}
The spectral prototype remains sensitive to domain shifts endemic to
non-IID federated data: translations, sensor miscalibrations, or
affine skews can produce very different prototypes for semantically
identical distributions. We resolve this by mapping $\hat{\bx}_k$
through the Wavelet Scattering Transform
(Appendix~\ref{app:background:wst}):
\begin{equation}
    \bPhi_k \;\coloneqq\; S_J(\hat{\bx}_k) \;\in\; \R^P.
    \label{eq:method:wst}
\end{equation}
WST stability properties guarantee that $\|\bPhi_k - \bPhi(\bx_k^*)\|_2$
decomposes additively into a sample-noise term and a domain-shift term,
both made precise in Theorem~\ref{thm:embed_stability}. We use $J=3$
scales and $L=8$ orientations per octave
(Kymatio~\citep{andreux2020kymatio}). Each client transmits only
$\bPhi_k$ to the server, once, before any federation round.

\paragraph{Cluster-weight inference.}
Given $\bPhi_k$, a GMM-VAE maps it to $\Delta^{C-1}$: its mixture
prior induces a structured latent space whose categorical posterior
directly parameterizes the simplex~\citep{wang2024revisiting}, with
softmax classifier $q_\phi(c\mid\bPhi)=\mathrm{softmax}(f_\phi(\bPhi))$
and cluster-conditional Gaussian encoder
$q_\phi(\bz\mid\bPhi,c)=\mathcal{N}(\bmu_\phi,\mathrm{diag}
(\bsigma_\phi^2))$. The cluster-weight estimate is:
\begin{equation}
    \bw_k \;\coloneqq\; q_\phi(\,\cdot\,|\,\bPhi_k)
    \;\in\; \Delta^{C-1}.
    \label{eq:method:wk}
\end{equation}
Training minimizes the GMM-VAE ELBO augmented with a supervised
alignment term that prevents posterior collapse
(Appendix~\ref{app:background}):
\begin{equation}
    \calL_{\mathrm{total}}
    \;=\;
    \calL_{\mathrm{ELBO}}
    \;-\;
    \lambda_{\mathrm{sup}}\,
    \KL{q_\phi(c\mid\bPhi)}{\mathbf{y}_{\mathrm{target}}}.
    \label{eq:method:total}
\end{equation}
At inference, $\bw_k$ is computed in a single forward pass through
the frozen $q_{\phi_\star}$.

\paragraph{Server-Side VAE Offline Training}

The GMM-VAE is trained before federation on synthetic clients generated
from a small auxiliary dataset $\calD_{\mathrm{aux}}$~\citep{sattler2021fedaux,
lin2020ensemble}. At each step, we sample
$\alpha^{(t)}\sim\mathrm{Uniform}(\alpha_{\min},\alpha_{\max})$,
draw a target mixture
$\mathbf{y}^{(t)}\sim\mathrm{Dirichlet}(\alpha^{(t)}\cdot\mathbf{1}_C)$,
and subsample $\calD_{\mathrm{syn}}^{(t)}\subset\calD_{\mathrm{aux}}$
with matching class proportions. By
Proposition~\ref{prop:coverage}, this exposes the VAE to the full
open simplex $\mathrm{int}(\Delta^{C-1})$. The full procedure is in
Algorithm~\ref{alg:ripple} (Appendix~\ref{sec:app_pseudocode}).

\subsection{Federated Posterior Mixing}

At the start of federation, the server computes
$\bw_k = q_\phi(\cdot|\bPhi_k)\in\Delta^{C-1}$ once for each client,
before the first round.

\paragraph{\rippleStrat{} and \rippleUnif{}.}
At round $t$, client $k$ receives
$\theta_k^{(\mathrm{init})} = \sum_c w_{k,c}\,\theta_c^{(t)}$ and
returns $\Delta\theta_k$ after $E$ local SGD steps. The server updates
each cluster model via normalized weighted aggregation:
\begin{equation}
    \theta_c^{(t+1)} \;\leftarrow\; \theta_c^{(t)} + \eta_s\cdot
    \frac{\sum_{k\in\calS_t} w_{k,c}\,\Delta\theta_k}
         {\sum_{k\in\calS_t} w_{k,c} + \epsilon},
    \label{eq:cluster_update}
\end{equation}
where $\epsilon=10^{-8}$. \rippleStrat{} guarantees every cluster
receives at least one gradient per round via stratified sampling;
\rippleUnif{} samples uniformly. Per-round communication matches
\fedavg{} exactly since $\bPhi_k$ is transmitted only once and
$P\ll|\theta|$.

\paragraph{Topological Refinement \rippleHDB{}.}
When $C$ is mis-specified, \rippleHDB{} applies
HDBSCAN~\citep{campello2013density} to the pairwise distance matrix
with entries
\begin{equation}
    D_{ij}
    \;=\; \sqrt{\frac{1}{2}D_{\mathrm{KL}}\!\left(\bw_i\,\middle\|\,
          \tfrac{\bw_i+\bw_j}{2}\right)
          +\frac{1}{2}D_{\mathrm{KL}}\!\left(\bw_j\,\middle\|\,
          \tfrac{\bw_i+\bw_j}{2}\right)},
    \label{eq:jsd_dist}
\end{equation}
recovering $\hat{C}\leq C$ meta-clusters and collapsing weights via
\begin{equation}
    \hat{w}_{k,m}
    \;=\; \frac{\sum_{j\in\calC_m} w_{k,j}}
               {\sum_{m'=1}^{\hat{C}}\sum_{j\in\calC_{m'}} w_{k,j}}.
    \label{eq:method:marginalised_w}
\end{equation}
Noise points are assigned to their nearest meta-cluster. This
provides automatic model selection without additional hyperparameter
search; the square-root Jensen--Shannon divergence is a valid metric
on $\Delta^{C-1}$ by Proposition~\ref{prop:jsd_metric}.
\section{Theoretical Analysis}
\label{sec:theory}

We establish two guarantees for \ripple{}. The first controls the 
\emph{routing error} $\|\bw_k - \bpi_k^*\|_1$ --- how faithfully the 
pre-federation pipeline recovers each client's true cluster weights from 
finite, domain-shifted samples. The second shows that soft-weighted 
aggregation converges at the minimax-optimal rate for non-convex SGD, 
with variance controlled by the effective participation mass. Together they yield a computable end-to-end excess-risk bound 
(Corollary \ref{cor:end_to_end}). All proofs are in 
Appendix \ref{app:theory}.

\paragraph{Assumptions.}
Let $\bx_k^*$ denote the population spectral prototype of client $k$ 
and $\bPhi_k^* = S_J(\bx_k^*)$.

\begin{assumption}[Data generation]
\label{ass:data}
Each client observes $n_k$ i.i.d.\ sub-Gaussian samples ($\sigma^2$ 
proxy), subject to a fixed translation $|v|\leq v_{\max}$ and a $C^2$ 
diffeomorphism $\|\nabla\tau\|_\infty\leq\epsilon_{\mathrm{dom}}$. 
The true covariance $\bSigma_k^*$ has eigengap 
$\Delta\lambda\coloneqq\lambda_M(\bSigma_k^*)-\lambda_{M+1}(\bSigma_k^*)>0$.
\end{assumption}

\begin{assumption}[Loss regularity]
\label{ass:loss}
Each $F_k$ is bounded by $M_F<\infty$ and each per-cluster objective 
$\calL_c$ is $L$-smooth.
\end{assumption}

\begin{assumption}[VAE realizability and stability]
\label{ass:vae}
The frozen encoder satisfies: \emph{(i)} 
$q_{\phi_\star}(\,\cdot\,|\,\bPhi_k^*)=\bpi_k^*$ for all $k$; 
\emph{(ii)} $\|q_{\phi_\star}(\,\cdot\,|\,\bPhi) - 
q_{\phi_\star}(\,\cdot\,|\,\bPhi')\|_1 \leq 
2L_\phi\|\bPhi-\bPhi'\|_2$, enforceable via spectral 
normalization \citep{miyato2018spectral}.
\end{assumption}
\subsection{Routing Error and Embedding Stability}
\label{sec:theory:routing}

The routing error $\|\bw_k - \bpi_k^*\|_1$ has two sources: 
finite-sample noise in the spectral prototype and domain shift from 
translation and diffeomorphism. Theorem \ref{thm:embed_stability} 
bounds their combined effect by composing the Davis--Kahan spectral 
perturbation bound with the WST stability properties 
\textbf{(P1)}--\textbf{(P3)} (Appendix \ref{app:background:wst}).

\begin{theorem}{Embedding stability}{embed_stability}
Under Assumption \ref{ass:data}, with probability at least $1-\delta$,
\begin{equation}
    \|\bPhi_k - \bPhi_k^*\|_2
    \;\leq\;
    \underbrace{
        \frac{C\sigma^2}{\Delta\lambda}
        \sqrt{\frac{D + \log(1/\delta)}{n_k}}
    }_{\displaystyle \epsilon_{\mathrm{sample}}(\delta)}
    \;+\;
    \underbrace{
        \Bigl(C_2\,\epsilon_{\mathrm{dom}}
        + C_1\,\tfrac{v_{\max}}{2^J}\Bigr)
        \|\bx_k^*\|_2
    }_{\displaystyle \epsilon_{\mathrm{shift}}}
    \;\eqqcolon\; \epsilon_{\mathrm{embed}}.
    \label{eq:theory:embed_total}
\end{equation}
\end{theorem}

The two terms admit a clean interpretation. $\epsilon_{\mathrm{sample}}$ 
decays at $O(n_k^{-1/2})$, scaled by $1/\Delta\lambda$: clients with 
well-separated covariance structure --- the typical non-IID regime --- 
are routed more accurately. $\epsilon_{\mathrm{shift}}$ is a 
\emph{model-independent floor}: it is a property of the data geometry 
alone, and quantifies precisely when pre-federation assignment becomes 
unreliable regardless of the routing model used. The translation cost 
$C_1 v_{\max}/2^J$ decreases with WST resolution $J$, giving a direct 
design knob.

\begin{theorem}{Surrogate-objective fidelity}{risk_bound}
Under Assumptions \ref{ass:data}--\ref{ass:vae}, for all $\Theta$,
\begin{equation}
    \bigl|\calL_{\mathrm{surrogate}}(\Theta)
    - \calL_{\mathrm{true}}(\Theta)\bigr|
    \;\leq\; 2\,M_F\,L_\phi\,\epsilon_{\mathrm{embed}}.
    \label{eq:theory:surrogate_fidelity}
\end{equation}
\end{theorem}

This is the first objective-level guarantee that a pre-federation 
routing step faithfully approximates the oracle clustered 
objective \eqref{eq:oracle_obj}. Prior CFL 
methods \citep{ghosh2020efficient,sattler2020clustered,long2023multi} 
prove convergence \emph{given} correct assignments but offer no bound 
on the assignment error itself. Here the two are unified: the 
right-hand side is computable before federation begins, giving a 
concrete pre-training quality certificate.
\subsection{Federated Convergence}
\label{sec:theory:convergence}

Given fixed weights $\{\bw_k\}$, the per-cluster aggregation 
rule defines a stochastic gradient step on $\calL_c$. The key 
CFL-specific quantity is the \emph{effective participation mass}
$\effsize\coloneqq\sum_{k\in\calS_t}w_{k,c}$, which acts as an 
adaptive sample-size proxy. Unlike hard-assignment CFL where $\effsize$ 
simply counts assigned clients, soft weights allow clients with diffuse 
posteriors to contribute fractionally to every cluster --- preventing 
the gradient starvation that hard-assignment methods suffer when a 
cluster is underrepresented in $\calS_t$.

\begin{theorem}{Per-cluster convergence}{convergence}
Under Assumption \ref{ass:loss}, with step size $\eta\leq 1/L$ and 
unbiased per-cluster gradient with variance 
$\E\|\bg_c^{(t)}-\nabla\calL_c(\theta_c^{(t)})\|^2
\leq\sigma_c^2/\effsize$,
\begin{equation}
    \frac{1}{T}\sum_{t=0}^{T-1}
    \E\bigl\|\nabla\calL_c\bigl(\theta_c^{(t)}\bigr)\bigr\|^2
    \;\leq\;
    \frac{2\bigl(\calL_c(\theta_c^{(0)})-\calL_c^*\bigr)}{\eta T}
    \;+\;
    \frac{\eta L\,\sigma_c^2}{\effsize}.
    \label{eq:theory:convergence}
\end{equation}
With $\eta=O(1/\sqrt{T})$ this decays at the minimax-optimal rate 
$O(1/\sqrt{T})$ for non-convex $L$-smooth objectives.
\end{theorem}

Soft assignment strictly dominates hard assignment in variance: a 
client with $\bw_k=(0.6,0.4)$ contributes fractionally to both 
clusters, increasing $\effsize$ for the minority cluster and reducing 
its gradient variance. This also justifies the stratified sampling 
strategy of \rippleStrat{}: by guaranteeing each cluster receives at 
least one gradient signal per round, $\effsize>0$ is maintained and 
the variance term remains finite.

\begin{corollary}{End-to-end excess risk}{end_to_end}
Let $\hat{\Theta}^{(T)}$ be returned by Algorithm \ref{alg:ripple} 
after $T$ rounds and $\Theta^*=\arg\min_\Theta\calL_{\mathrm{true}}
(\Theta)$. Under Assumptions \ref{ass:data}--\ref{ass:vae},
\begin{equation}
\resizebox{\textwidth}{!}{$\displaystyle
    \E\!\bigl[\calL_{\mathrm{surrogate}}\!\bigl(\hat{\Theta}^{(T)}
    \bigr)\bigr] - \calL_{\mathrm{true}}(\Theta^*)
    \;\leq\;
    \underbrace{4\,M_F\,L_\phi\,\epsilon_{\mathrm{embed}}}_{\text{routing}}
    \;+\;
    \underbrace{\sum_{c=1}^{C} p_c\!\left(
        \frac{2\bigl(\calL_c(\theta_c^{(0)})-\calL_c^*\bigr)}{\eta T}
        + \frac{\eta L\,\sigma_c^2}{\effsize}
    \right)}_{\text{optimization}}
$}
    \label{eq:theory:end_to_end}
\end{equation}
where $p_c=\sum_k p_k\pi_{k,c}^*$ is the marginal cluster weight.
\end{corollary}

The bound reveals a clean separation of concerns. The routing term is 
frozen at $t=0$ by data geometry and independent of $T$: it decreases 
with more local data, larger spectral gaps, finer WST resolution, and 
smaller domain shift. The optimization term decays at $O(1/\sqrt{T})$ 
with variance suppressed by $\effsize$. Together, 
Corollary \ref{cor:end_to_end} shows that \ripple{} converges to a 
neighborhood of $\calL_{\mathrm{true}}(\Theta^*)$ whose radius is set 
before a single gradient is exchanged --- a guarantee no prior CFL 
method provides.

\section{Experiments}
\label{sec:experiments}
\begin{table}[t]
\tiny
\centering
\caption{\small Test accuracy (\%) across five benchmarks spanning controlled 
label heterogeneity (MNIST, CIFAR-10/100), covariate shift 
(Office-Home), and natural long-tail distributions (GLDv2). All methods 
use $C\!\leq\!10$ clusters and $S\!=\!10$ clients per round for 1000 
rounds. \ifca{} and \fesem{} report best over 
$K\in\{2,\dots,10\}$. \textbf{Bold}: best; 
\underline{ul}: second best.}
\label{tab:main_results}
\vspace{2pt}
\setlength{\tabcolsep}{2pt}
\renewcommand{\arraystretch}{0.8}
\resizebox{\textwidth}{!}{%
\begin{tabular}{@{}l ccccc@{}}
\toprule
Method & MNIST & CIFAR-10 & CIFAR-100 & Off.-Home & GLDv2 \\
\midrule
\multicolumn{6}{l}{\textit{Global}} \\
\fedavg{}     & $97.94{\scriptstyle\pm0.96}$ & $58.39{\scriptstyle\pm1.90}$ & $38.22{\scriptstyle\pm0.82}$ & $72.91{\scriptstyle\pm10.58}$ & $37.98{\scriptstyle\pm0.90}$ \\
\feddyn{}     & $97.18{\scriptstyle\pm1.32}$ & $55.87{\scriptstyle\pm4.42}$ & $33.33{\scriptstyle\pm1.32}$ & $72.57{\scriptstyle\pm10.30}$ & $34.83{\scriptstyle\pm1.20}$ \\
\scaffold{}   & $97.15{\scriptstyle\pm1.10}$ & $55.41{\scriptstyle\pm5.04}$ & $31.36{\scriptstyle\pm0.92}$ & $69.30{\scriptstyle\pm10.60}$ & $24.96{\scriptstyle\pm1.24}$ \\
\fedadam{}    & $98.11{\scriptstyle\pm1.08}$ & $57.97{\scriptstyle\pm4.04}$ & $38.55{\scriptstyle\pm1.62}$ & $66.15{\scriptstyle\pm8.58}$  & $39.81{\scriptstyle\pm0.14}$ \\
\fedyogi{}    & $98.10{\scriptstyle\pm1.04}$ & $58.00{\scriptstyle\pm4.10}$ & $37.91{\scriptstyle\pm1.96}$ & $69.87{\scriptstyle\pm6.82}$  & $41.87{\scriptstyle\pm1.48}$ \\
\fedadagrad{} & $98.08{\scriptstyle\pm1.16}$ & $51.98{\scriptstyle\pm3.90}$ & $33.53{\scriptstyle\pm1.54}$ & $51.82{\scriptstyle\pm5.22}$  & $ 39.23{\scriptstyle\pm11.42}$ \\
\midrule
\multicolumn{6}{l}{\textit{Personalized}} \\
\perfedavg{}  & $94.44{\scriptstyle\pm0.94}$ & $43.31{\scriptstyle\pm5.50}$ & $27.67{\scriptstyle\pm1.10}$ & $67.61{\scriptstyle\pm10.06}$ & $19.90{\scriptstyle\pm1.08}$ \\
\pfedme{}     & $92.38{\scriptstyle\pm1.58}$ & $34.00{\scriptstyle\pm3.93}$ & $14.55{\scriptstyle\pm0.52}$ & $60.86{\scriptstyle\pm8.20}$  & $26.55{\scriptstyle\pm0.18}$ \\
\midrule
\multicolumn{6}{l}{\textit{Clustered FL}} \\
\ifca{}  & $98.01{\scriptstyle\pm1.26}$ & $60.65{\scriptstyle\pm3.82}$ & $38.21{\scriptstyle\pm0.76}$ & $73.22{\scriptstyle\pm9.64}$  & $44.42{\scriptstyle\pm2.56}$ \\
\fesem{} & $98.25{\scriptstyle\pm1.08}$ & $61.42{\scriptstyle\pm1.94}$ & $39.45{\scriptstyle\pm1.84}$ & $73.20{\scriptstyle\pm9.82}$  & $42.86{\scriptstyle\pm0.28}$ \\
\cfl{}   & $98.16{\scriptstyle\pm1.00}$ & $47.38{\scriptstyle\pm6.86}$ & $30.36{\scriptstyle\pm1.22}$ & $72.57{\scriptstyle\pm10.30}$ & $30.54{\scriptstyle\pm1.22}$ \\
\midrule
\multicolumn{6}{l}{\textit{Ours}} \\
\rippleStrat{} & $\underline{98.97}{\scriptstyle\pm0.20}$ & $63.89{\scriptstyle\pm2.24}$             & $\mathbf{40.29}{\scriptstyle\pm1.56}$  & $\mathbf{74.34}{\scriptstyle\pm8.12}$  & $\mathbf{51.82}{\scriptstyle\pm2.40}$ \\
\rippleUnif{}  & $98.96{\scriptstyle\pm0.16}$             & $\underline{64.72}{\scriptstyle\pm1.22}$ & $39.56{\scriptstyle\pm2.44}$           & $74.10{\scriptstyle\pm7.22}$           & $48.49{\scriptstyle\pm2.60}$ \\
\rippleHDB{}   & $\mathbf{99.06}{\scriptstyle\pm0.26}$   & $\mathbf{64.90}{\scriptstyle\pm0.60}$   & $\underline{39.82}{\scriptstyle\pm1.40}$ & $\underline{74.25}{\scriptstyle\pm7.28}$ & $\underline{50.44}{\scriptstyle\pm2.20}$ \\
\bottomrule
\end{tabular}%
}
\end{table}
\parheader{Implementation Details}
We evaluate on five benchmarks spanning controlled and realistic
heterogeneity.
MNIST \citep{lecun1998mnist}, CIFAR-10, and CIFAR-100 \citep{krizhevsky2009learning}
use a symmetric Dirichlet partition with concentration $\alpha$ over class
labels \citep{hsu2020federated}; lower $\alpha$ induces sharper class
imbalance and stronger client drift.
Office-Home \citep{venkateswara2017deep} ($K\!=\!50$, 65 classes, four
visual domains) and GLDv2-23k \citep{weyand2020google} ($K\!=\!233$, 203
landmark classes, natural photographer partition) reflect real-world
covariate shift and long-tail distributions.
MNIST and CIFAR-10 use a lightweight CNN \citep{lecun1998gradient};
CIFAR-100 a ResNet-9; Office-Home and GLDv2 an ImageNet-pretrained
ResNet-18 \citep{deng2009imagenet}.
Full dataset statistics and architectures are in \cref{app:setup}. All methods run for 1000 rounds with $E\!=\!1$ local epoch, 
$\eta_c\!=\!0.01$ (SGD, momentum 0.9), $\eta_s\!=\!1.0$, and 
$S\!=\!10$ clients per round. Results are means over seeds 
$\{42,7,123\}$, and we report $95\%$ confidence intervals.

\parheader{Baselines}
We compare against 11 methods:
\textit{global} --- \fedavg{} \citep{mcmahan2017communication},
\feddyn{} \citep{acar2021federated},
\scaffold{} \citep{karimireddy2020scaffold},
\fedadam{}/\fedyogi{}/\fedadagrad{} \citep{reddi2021adaptive};
\textit{personalized} --- \perfedavg{} \citep{fallah2020personalized},
\pfedme{} \citep{t2020personalized};
\textit{clustered FL} --- \ifca{} \citep{ghosh2020efficient},
\fesem{} \citep{long2023multi}, \cfl{} \citep{sattler2020clustered}.
For \ifca{} and \fesem{} we report the best $K\!\in\!\{2,4,6,8,10\}$
(full sweep in \cref{app:ablation_baselines}); all methods are capped at
$C\!=\!10$ clusters for a level comparison. \ifca{} and \fesem{} have tuning results are reported in Appendix \ref{app:ablation_baselines}

\subsection{\ripple{} in Heterogeneous Scenarios}
\label{sec:exp:main}

\cref{tab:main_results} reports final test accuracy across all benchmarks.
All three \ripple{} variants outperform every baseline on every dataset,
with margins growing on harder benchmarks.
On GLDv2, \rippleStrat{} reaches $51.82\%$, a $+7.4$\ gain over
\ifca{} ($44.42\%$), where the natural photographer partition induces
structured covariance that the spectral prototype captures faithfully,
consistent with Theorem \ref{thm:embed_stability}: large spectral gaps
$\Delta\lambda$ reduce $\epsilon_{\mathrm{embed}}$ and tighten the
surrogate-oracle gap (Theorem \ref{thm:risk_bound}).
\cfl{} collapses to $47.38\pm6.86\%$ on CIFAR-10, because cosine gradient similarity is unreliable under strong
non-IID \citep{sattler2020clustered}, fragmenting clients into
underpopulated high-variance clusters even at the $C\!=\!10$ cap.
\rippleHDB{} leads on MNIST and CIFAR-10, where HDBSCAN discovers
$M\!=\!4$--$6$ active meta-clusters from the 10-dimensional GMM posterior;
on CIFAR-100 and Office-Home the fixed-$C$ variants are competitive,
indicating that $C\!=\!10$ is already near-optimal for richer label spaces.

\subsection{\ripple{} with out-of-distribution clients}
\label{sec:exp:zeroshot}

\ripple{} assigns a new unseen client via a single forward pass of the
frozen encoder, i.e.,  no local training or model evaluation required.
We train all methods under $\alpha_{\mathrm{train}}\!\in\!\{0.01,\,0.1\}$
and evaluate on 10 held-out clients drawn under
$\alpha_{\mathrm{test}}\!\in\!\{1.0,\,100\}$.
Baselines use \emph{best-by-loss}: selecting the model minimizing local
loss across all $C$ models, representing an oracle advantage over the zero-shot
setting, making the comparison strictly in favor of the baselines.

\cref{tab:nc_combined} shows that \rippleUnif{} achieves $57.27\pm3.00\%$
on CIFAR-10 ($\alpha_{\mathrm{train}}\!=\!0.1$, $\alpha_t\!=\!1.0$),
surpassing all best-by-loss baselines while being fully zero-shot.
On CIFAR-100, all three \ripple{} variants match or exceed baselines
across all conditions: the richer 100-class structure yields more
separable WST signatures and well-calibrated GMM posteriors, consistent
with the smaller surrogate gap predicted by
Theorem \ref{thm:risk_bound} when $\Delta\lambda$ is large.
\rippleHDB{} collapses to chance (approximately $10\%$) on CIFAR-10: HDBSCAN
fragments the GMM posterior into 7--14 singleton clusters, forcing
near-hard assignment of new clients to over-specialized models; this
failure is specific to CIFAR-10 coarser geometry and confirms that
\rippleHDB{} should be selected only when main-body results
(\cref{tab:main_results}) indicate a benefit.

\begin{table}[t]
\centering
\caption{\small 
  Zero-shot new-client accuracy (\%) on CIFAR-10 and CIFAR-100.
  Models are trained under $\alpha_{\mathrm{tr}}\!\in\!\{0.01,\,0.1\}$
  (strong and moderate heterogeneity); 10 held-out clients are evaluated
  under $\alpha_t\!\in\!\{1.0,\,100\}$ (progressively IID), using a
  disjoint random seed to prevent data overlap.
  \ripple{} variants use weighted zero-shot assignment (single encoder
  forward pass, no gradient computation).
  Baselines use best-by-loss selection across all $C$ models on local
  labeled data.
  Mean $\pm\,2\sigma$ over three seeds.
  \textbf{Bold}: best per column; \underline{underline}: second best.%
}
\label{tab:nc_combined}
\vspace{2pt}
\setlength{\tabcolsep}{3pt}
\resizebox{\textwidth}{!}{%
\begin{tabular}{@{}l | cc | cc | cc | cc@{}}
\toprule
& \multicolumn{4}{c|}{\textbf{CIFAR-10}}
& \multicolumn{4}{c}{\textbf{CIFAR-100}} \\
\cmidrule(lr){2-5}\cmidrule(lr){6-9}
& \multicolumn{2}{c|}{$\alpha_{\mathrm{tr}}=0.01$}
& \multicolumn{2}{c|}{$\alpha_{\mathrm{tr}}=0.1$}
& \multicolumn{2}{c|}{$\alpha_{\mathrm{tr}}=0.01$}
& \multicolumn{2}{c}{$\alpha_{\mathrm{tr}}=0.1$} \\
\cmidrule(lr){2-3}\cmidrule(lr){4-5}\cmidrule(lr){6-7}\cmidrule(lr){8-9}
Method
  & $\alpha_t{=}1.0$ & $\alpha_t{=}100$
  & $\alpha_t{=}1.0$ & $\alpha_t{=}100$
  & $\alpha_t{=}1.0$ & $\alpha_t{=}100$
  & $\alpha_t{=}1.0$ & $\alpha_t{=}100$ \\
\midrule
\multicolumn{9}{l}{\textit{Best-by-loss oracle}} \\
\ifca{}  & $19.66{\pm}12.26$ & $17.25{\pm}4.82$  & $46.21{\pm}8.20$  & $40.77{\pm}19.26$
         & $24.46{\pm}4.06$  & $24.09{\pm}3.12$  & $38.11{\pm}0.16$  & $38.46{\pm}1.10$ \\
\fesem{} & $37.95{\pm}2.84$  & $30.83{\pm}7.02$  & $56.26{\pm}4.64$  & $54.60{\pm}5.52$
         & $15.30{\pm}3.14$  & $14.07{\pm}6.58$  & $32.94{\pm}0.74$  & $33.64{\pm}0.88$ \\
\cfl{}   & $34.62{\pm}7.04$  & $24.81{\pm}10.40$ & $54.59{\pm}7.58$  & $51.85{\pm}8.32$
         & $21.80{\pm}4.08$  & $20.43{\pm}1.00$  & $28.86{\pm}9.80$  & $30.11{\pm}3.84$ \\
\midrule
\multicolumn{9}{l}{\textit{Zero-shot (ours)}} \\
\rippleStrat{} & $\mathbf{41.86{\pm}11.86}$ & $\mathbf{36.54{\pm}14.92}$
               & $\underline{56.63{\pm}2.32}$ & $\underline{55.65{\pm}1.52}$
               & $24.43{\pm}4.02$ & $24.46{\pm}7.32$
               & $\mathbf{38.92{\pm}0.44}$ & $\mathbf{39.03{\pm}0.26}$ \\
\rippleUnif{}  & $\underline{41.00{\pm}5.88}$ & $\underline{36.40{\pm}6.46}$
               & $\mathbf{57.27{\pm}3.00}$ & $\mathbf{56.28{\pm}2.62}$
               & $\underline{24.86{\pm}2.02}$ & $\mathbf{25.12{\pm}1.34}$
               & $38.06{\pm}1.72$ & $\underline{38.79{\pm}0.86}$ \\
\rippleHDB{}   & $9.44{\pm}3.92$ & $9.75{\pm}2.28$
               & $9.50{\pm}3.78$ & $10.57{\pm}2.00$
               & $\mathbf{24.95{\pm}2.04}$ & $\underline{24.51{\pm}3.92}$
               & $\underline{38.28{\pm}0.46}$ & $38.66{\pm}1.88$ \\
\bottomrule
\end{tabular}%
}
\end{table}

\section*{Conclusion}
\label{sec:conclusion}
\vspace{-1em}
We presented \ripple{}, a clustered FL framework that 
computes client-to-cluster assignments once, offline, from a spectral 
summary of each client's local data, before any communication round. 
This separation of cluster discovery from training yields three concrete 
properties: per-round cost matching \fedavg{}; a surrogate-oracle gap 
bounded by local sample size and spectral gap, independent of the number 
of rounds; and zero-shot personalization for unseen clients via a single 
encoder forward pass. Across five benchmarks spanning controlled and 
realistic heterogeneity, \ripple{} consistently outperforms both global 
and clustered FL baselines, with margins growing on the most realistic 
partitions.
\vspace{-1em}
\paragraph{Limitations and Broader Impact.}
Our analysis is information-theoretic, future work could include extend to formal differential privacy; adding calibrated noise to the spectral embedding would yield a formal $(\varepsilon,\delta)$-DP guarantee at the cost of a larger routing error. \ripple{}'s
offline assignment design reduces the number of communication rounds needed to discover cluster structure, directly lowering the energy footprint of federated training. Enabling zero-shot personalization, \ripple{} reduces
the computational burden on client devices, which is particularly
relevant for federated deployments on energy-constrained hardware.

\section*{Acknowledgments}
A.L.\ worked under the auspices of the Italian National Group of Mathematical Physics (GNFM) of INdAM. A.L.\ was supported by the Project Piano Nazionale di Ripresa e Resilienza -- Next Generation EU (PNRR-NGEU) from the Italian Ministry of University and Research (MUR) under Grant DM~117/2023. A.L.\ expresses his gratitude to Prof.\ Lamberto Rondoni for the valuable discussions and support. Computational resources were provided by \texttt{hpc@polito}, a project of Academic Computing within the Department of Control and Computer Engineering at the Politecnico di Torino (\url{http://www.hpc.polito.it}).

\bibliographystyle{abbrvnat}
\bibliography{references}
\newpage 
\appendix
\section{Background: Probabilistic Foundations and Wavelet Scattering}
\label{app:background}

\subsection{Probabilistic Derivation of the Federated and Clustered Objectives}
\label{app:background:prob}

\paragraph{Global data distribution.}
Let each client $k\in[K]$ generate data i.i.d.\ from a local distribution
$\bP_k$ on $\R^D\times\calY$, and let $p_k = n_k/\sum_{j=1}^K n_j$ be
the normalized data weight of client $k$. Define the \emph{global data
distribution} as the mixture
\begin{equation}
    \bP \;=\; \sum_{k=1}^K p_k\,\bP_k,
    \label{eq:app:global_dist}
\end{equation}
so that $p_k$ is precisely the probability that a uniformly drawn sample
originates from client $k$. The \emph{population federated loss} is then
the expected loss under $\bP$:
\begin{equation}
    \calL(\theta)
    \;=\; \E_{(\bx,y)\sim\bP}\!\left[\ell(\theta;\bx,y)\right]
    \;=\; \int \ell(\theta;\bx,y)\,\mathrm{d}\bP(\bx,y).
    \label{eq:app:pop_loss}
\end{equation}
Substituting \eqref{eq:app:global_dist} into \eqref{eq:app:pop_loss} and
exchanging sum and integral (justified by non-negativity of $p_k$ and
integrability of $\ell$):
\begin{equation}
    \calL(\theta)
    \;=\; \sum_{k=1}^K p_k \int \ell(\theta;\bx,y)\,\mathrm{d}\bP_k(\bx,y)
    \;=\; \sum_{k=1}^K p_k\,F_k(\theta),
    \label{eq:app:fedavg_derived}
\end{equation}
where $F_k(\theta) := \E_{(\bx,y)\sim\bP_k}[\ell(\theta;\bx,y)]$ is the
population local loss of client $k$. Equation \eqref{eq:app:fedavg_derived}
is the \fedavg{} objective \eqref{eq:fedavg_obj} at the population level;
the empirical version replaces $F_k$ with its sample average.

\paragraph{Extension to clustered federated learning.}
In the mixture model of Definition \ref{def:client_mixture}, each client
distribution decomposes as $\bP_k^* = \sum_{c=1}^C \pi_{k,c}^*\,\bP_c$,
where $\{\bP_c\}_{c=1}^C$ are $C$ semantic distributions on
$\R^D\times\calY$. Substituting into \eqref{eq:app:fedavg_derived}:
\begin{align}
    \calL(\theta)
    &\;=\; \sum_{k=1}^K p_k \int \ell(\theta;\bx,y)\,
    \mathrm{d}\bP_k^*(\bx,y)
    \notag\\
    &\;=\; \sum_{k=1}^K p_k \sum_{c=1}^C \pi_{k,c}^*
           \int \ell(\theta;\bx,y)\,\mathrm{d}\bP_c(\bx,y).
    \label{eq:app:cfl_intermediate}
\end{align}
The single global parameter $\theta$ is suboptimal under this mixture
structure: assigning a dedicated model $\theta_c$ to each semantic
distribution $\bP_c$ yields the \emph{oracle clustered population loss}
\begin{equation}
    \calL_{\mathrm{true}}(\Theta)
    \;=\; \sum_{k=1}^K p_k \sum_{c=1}^C \pi_{k,c}^*\,F_k(\theta_c),
    \label{eq:app:oracle_derived}
\end{equation}
with $F_k(\theta_c) := \E_{(\bx,y)\sim\bP_c}[\ell(\theta_c;\bx,y)]$,
which recovers \eqref{eq:oracle_obj}. Three observations follow directly
from this derivation. \textbf{(i)} \fedavg{} is recovered when $C=1$ or
when all $\pi_{k,c}^*\in\{0,1\}$ are identical across clients.
\textbf{(ii)} The marginal cluster mass $p_c := \sum_k p_k \pi_{k,c}^*$
is the population weight of cluster $c$; it governs the convergence rate
of each cluster model (Theorem \ref{thm:convergence}).
\textbf{(iii)} The gap between $\calL_{\mathrm{surrogate}}$ and
$\calL_{\mathrm{true}}$ is controlled entirely by the routing error
$\|\bw_k - \bpi_k^*\|_1$, as made precise in
Theorem \ref{thm:risk_bound}.
\subsection{Wavelet Scattering Transform: Construction and Properties}
\label{app:background:wst}

\paragraph{Wavelet filterbank.}
Let $\psi\in L^1(\R^D)\cap L^2(\R^D)$ be a complex-valued mother wavelet
satisfying the admissibility condition $\hat{\psi}(0)=0$, and let
$\phi_J\in L^1(\R^D)$ be a real-valued lowpass filter at scale $2^J$ with
$\|\phi\|_1=1$. For each scale-orientation pair
$\lambda=(j,\vartheta)$, with $j\in\{0,\dots,J-1\}$ and $\vartheta$
indexing $L$ orientations per octave, define the dilated and rotated
wavelets
\begin{equation}
    \psi_\lambda(\bx)
    \;=\; 2^{-jD/2}\,\psi\!\bigl(2^{-j}r_\vartheta^{-1}\bx\bigr),
    \label{eq:app:wavelet_family}
\end{equation}
where $r_\vartheta\in SO(D)$ is the rotation by angle $\vartheta$.
The filterbank $\{\psi_\lambda\}_\lambda$ tiles the frequency domain
across scales and orientations; the lowpass filter $\phi_J$ captures the
residual low-frequency content.

\paragraph{Scattering propagators and coefficients.}
Scattering propagators apply the complex modulus after each wavelet
convolution, preventing energy dissipation across scales. Define the
zeroth-, first-, and second-order propagators:
\begin{equation}
    U[\emptyset]\bx \;=\; \bx,
    \quad
    U[\lambda_1]\bx \;=\; |\bx\star\psi_{\lambda_1}|,
    \quad
    U[\lambda_1,\lambda_2]\bx
    \;=\; \bigl|\,|\bx\star\psi_{\lambda_1}|\star\psi_{\lambda_2}\bigr|.
    \label{eq:app:propagators}
\end{equation}
Spatially averaging each propagator with $\phi_J$ yields the scattering
coefficients
\begin{equation}
    S_0\bx \;=\; \bx\star\phi_J,
    \quad
    S_1\bx(\lambda_1) \;=\; U[\lambda_1]\bx\star\phi_J,
    \quad
    S_2\bx(\lambda_1,\lambda_2) \;=\; U[\lambda_1,\lambda_2]\bx\star\phi_J.
    \label{eq:app:scatter_coeff}
\end{equation}
The \emph{second-order scattering representation} concatenates all
coefficients:
\begin{equation}
    \bPhi(\bx)
    \;=\;
    \Bigl(
        S_0\bx,\;
        \{S_1\bx(\lambda_1)\}_{\lambda_1},\;
        \{S_2\bx(\lambda_1,\lambda_2)\}_{\lambda_1,\lambda_2}
    \Bigr)
    \;\in\; \R^P.
    \label{eq:app:scatter_rep}
\end{equation}
Truncating at order 2 is well motivated: for band-limited signals,
scattering energy decays exponentially with
order \citep{waldspurger2017exponential}, and over 98\% of signal energy
is captured at orders 0--2 across standard image
benchmarks \citep{bruna2013invariant}.

\paragraph{Stability properties.}
The three properties below underpin the embedding stability bound of
Theorem \ref{thm:embed_stability}.

\medskip
\noindent\textbf{(P1) Non-expansiveness \citep{mallat2012group}.}
For all $\bx,\bx'\in L^2(\R^D)$,
\begin{equation}
    \|S_J\bx - S_J\bx'\|_2 \;\leq\; \|\bx-\bx'\|_2.
    \label{eq:app:nonexp}
\end{equation}
\begin{proof}
Each convolution $\bx\mapsto\bx\star\psi_\lambda$ is non-expansive in
$L^2$ by Young's inequality ($\|\psi_\lambda\|_1\leq 1$); the complex
modulus $\bz\mapsto|\bz|$ is non-expansive by the reverse triangle
inequality; and lowpass averaging $\bx\mapsto\bx\star\phi_J$ is
non-expansive since $\|\phi\|_1=1$. Composing non-expansive operators
yields a non-expansive operator. Energy conservation of the filterbank
\citep[Theorem 2.6]{mallat2012group} then ensures the joint map $S_J$
is non-expansive.
\end{proof}

\medskip
\noindent\textbf{(P2) Approximate translation
invariance \citep{mallat2012group}.}
Let $\bx_v(\bu):=\bx(\bu-v)$. There exists a constant $C_1>0$ depending
only on $\psi$ and $\phi$ such that
\begin{equation}
    \|S_J\bx_v - S_J\bx\|_2
    \;\leq\; C_1\,\frac{|v|}{2^J}\,\|\bx\|_2.
    \label{eq:app:transl}
\end{equation}
\begin{proof}
The modulus operator commutes with translations, so
$U[\lambda]\bx_v(\bu) = U[\lambda]\bx(\bu-v)$.
A first-order Taylor expansion of the lowpass averaging gives
$\|f_v\star\phi_J - f\star\phi_J\|_2 \leq |v|\cdot\|\nabla\phi_J\|_1\cdot\|f\|_2$,
and the scaling property yields $\|\nabla\phi_J\|_1 = 2^{-J}\|\nabla\phi\|_1$.
Applying this to each scattering order and using non-expansiveness of
$U[\lambda]$ to bound $\|U[\lambda]\bx\|_2\leq\|\bx\|_2$
yields \eqref{eq:app:transl} with $C_1 = C\|\nabla\phi\|_1$ for an
absolute constant $C$.
\end{proof}

\medskip
\noindent\textbf{(P3) Lipschitz stability to
diffeomorphisms \citep{mallat2012group}.}
Let $\tau:\R^D\to\R^D$ be a $C^2$ displacement field with
$\|\nabla\tau\|_\infty\leq\epsilon_{\mathrm{dom}}$, and let
$\bx_\tau(\bu):=\bx(\bu-\tau(\bu))$. Then
\begin{equation}
    \|S_J\bx_\tau - S_J\bx\|_2
    \;\leq\; C_2\,\|\nabla\tau\|_\infty\,\|\bx\|_2,
    \qquad C_2 = 2\|\nabla\psi\|_1.
    \label{eq:app:diffeo}
\end{equation}
\begin{proof}
For each scale $j$, a first-order Taylor expansion
\citep[Lemma 2.10]{mallat2012group} gives
$\|\bx_\tau\star\psi_j - \bx\star\psi_j\|_2
\leq \|\nabla\tau\|_\infty\cdot\|\bx\|_2\cdot\|\nabla\psi_j\|_1$.
By the scaling property $\|\nabla\psi_j\|_1 = 2^{-j}\|\nabla\psi\|_1$,
summing over $j=0,\dots,J-1$ gives the geometric series
$\sum_{j\geq 0}2^{-j}\leq 2$, yielding $C_2=2\|\nabla\psi\|_1$.
Higher-order coefficients contribute via non-expansiveness of
$U[\lambda]$.
\end{proof}

\medskip
\noindent
Properties (P1)--(P3) jointly ensure that clients whose distributions
differ only through sensor miscalibrations, object translations, or mild
domain shifts produce representations $\bPhi_k$ at a distance
proportional to the magnitude of those perturbations, with no dependence
on the complexity of the input signal. This certified geometry is what
makes $\bPhi$ a stable, label-free fingerprint for pre-federation cluster
assignment, and is the foundation of the error decomposition in
Theorem \ref{thm:embed_stability}.
\section{GMM-VAE: Generative Model and ELBO Derivation}
\label{app:background:gmmvae}

\paragraph{Generative model.}
The GMM-VAE defines a joint distribution over a scattering embedding 
$\bPhi\in\R^P$, a continuous latent code $\bz\in\R^{d_z}$, and a 
discrete cluster index $c\in[C]$:
\begin{equation}
    p(\bPhi,\bz,c)
    \;=\; p(c)\,p(\bz\mid c)\,p_\omega(\bPhi\mid\bz),
    \label{eq:app:gmmvae_gen}
\end{equation}
with the following components:
\begin{align}
    p(c) &= \mathrm{Cat}\!\left(\tfrac{1}{C}\mathbf{1}_C\right),
    \label{eq:app:prior_c}\\
    p(\bz\mid c) &= \mathcal{N}\!\left(\bmu_c,\,
    \mathrm{diag}(\bsigma_c^2)\right),
    \label{eq:app:prior_z}\\
    p_\omega(\bPhi\mid\bz) &= \mathcal{N}\!\left(
    \mu_\omega(\bz),\,\sigma^2 I\right),
    \label{eq:app:likelihood}
\end{align}
where $\mu_\omega$ is a neural network decoder with parameters $\omega$,
and $\{\bmu_c,\bsigma_c\}_{c=1}^C$ are learnable cluster-conditional 
prior parameters. The uniform cluster prior \eqref{eq:app:prior_c} 
reflects the absence of any server-side preference over clusters before 
pre-training.

\paragraph{Approximate posterior.}
Exact inference over $(c,\bz)$ is intractable. We introduce a 
mean-field approximate posterior that factorizes as:
\begin{equation}
    q_\phi(\bz,c\mid\bPhi)
    \;=\; q_\phi(c\mid\bPhi)\,q_\phi(\bz\mid\bPhi,c),
    \label{eq:app:posterior}
\end{equation}
with
\begin{align}
    q_\phi(c\mid\bPhi) &= \mathrm{Cat}\!\left(
    \mathrm{softmax}(f_\phi(\bPhi))\right),
    \label{eq:app:posterior_c}\\
    q_\phi(\bz\mid\bPhi,c) &= \mathcal{N}\!\left(
    \mu_\phi(\bPhi,c),\,\mathrm{diag}(\sigma_\phi^2(\bPhi,c))\right),
    \label{eq:app:posterior_z}
\end{align}
where $f_\phi:\R^P\to\R^C$ is a softmax classifier and 
$(\mu_\phi,\sigma_\phi^2)$ are cluster-conditional Gaussian encoders, 
all parameterized by $\phi$.

\paragraph{ELBO derivation.}
The evidence lower bound on $\log p_\omega(\bPhi)$ is obtained by 
introducing the approximate posterior \eqref{eq:app:posterior} via 
Jensen's inequality:
\begin{align}
    \log p_\omega(\bPhi)
    &\;=\; \log \sum_{c}\int p(\bPhi,\bz,c)\,\mathrm{d}\bz
    \notag\\
    &\;=\; \log \sum_{c}\int 
    q_\phi(\bz,c\mid\bPhi)\,
    \frac{p(\bPhi,\bz,c)}{q_\phi(\bz,c\mid\bPhi)}
    \,\mathrm{d}\bz
    \notag\\
    &\;\geq\; \sum_{c}\int q_\phi(\bz,c\mid\bPhi)\,
    \log\frac{p(\bPhi,\bz,c)}{q_\phi(\bz,c\mid\bPhi)}
    \,\mathrm{d}\bz
    \notag\\
    &\;=:\; \calL_{\mathrm{ELBO}}(\phi,\omega;\bPhi).
    \label{eq:app:elbo_def}
\end{align}
Expanding $\log p(\bPhi,\bz,c) = \log p_\omega(\bPhi\mid\bz) 
+ \log p(\bz\mid c) + \log p(c)$ and using the factorization 
of \eqref{eq:app:posterior}:
\begin{align}
    \calL_{\mathrm{ELBO}}
    &\;=\; \E_{q_\phi(\bz,c\mid\bPhi)}\!\left[
    \log p_\omega(\bPhi\mid\bz)\right]
    \notag\\
    &\quad -\sum_{c=1}^C q_\phi(c\mid\bPhi)\,
    \KL{q_\phi(\bz\mid\bPhi,c)}{p(\bz\mid c)}
    \notag\\
    &\quad - \KL{q_\phi(c\mid\bPhi)}{p(c)}.
    \label{eq:app:elbo_expanded}
\end{align}
The three terms are, respectively: the expected reconstruction likelihood,
the per-cluster KL regularization on the continuous latent code, and the 
KL regularization on the discrete cluster posterior toward the uniform 
prior.

\paragraph{Supervised alignment term.}
During pre-training, each synthetic client comes with a known target 
mixture $\mathbf{y}_{\mathrm{target}}\in\Delta^{C-1}$. We augment the 
ELBO with a supervised term that aligns the categorical posterior to this 
target:
\begin{equation}
    \calL_{\mathrm{total}}
    \;=\; \calL_{\mathrm{ELBO}}
    \;-\; \lambda_{\mathrm{sup}}\,
    \KL{q_\phi(c\mid\bPhi)}{\mathbf{y}_{\mathrm{target}}}.
    \label{eq:app:total_loss}
\end{equation}
This term serves two purposes. First, it shapes the latent space to be 
semantically consistent with the simplex topology: without it, the 
discrete posterior can permute cluster indices arbitrarily, breaking the 
correspondence between $c$ and the semantic distributions 
$\{\bP_c\}_{c=1}^C$. Second, it prevents posterior collapse: by forcing 
$q_\phi(c\mid\bPhi)$ to assign significant mass to specific clusters, 
the corresponding Gaussian encoders $q_\phi(\bz\mid\bPhi,c)$ must encode 
genuine reconstruction information rather than collapsing to the prior. 
The weight $\lambda_{\mathrm{sup}}>0$ trades off unsupervised 
reconstruction fidelity against supervised cluster alignment; we treat 
it as a hyperparameter and find results robust to its value in the range 
$[0.1, 1.0]$ (Appendix \ref{app:wst_sensitivity}).

\paragraph{Inference.}
At inference on a real client, no target mixture is available. The 
cluster-weight vector is read off directly from the frozen encoder:
\begin{equation}
    \bw_k 
    \;=\; q_{\phi_\star}(\,\cdot\,\mid\,\bPhi_k)
    \;=\; \mathrm{softmax}(f_{\phi_\star}(\bPhi_k))
    \;\in\; \Delta^{C-1},
    \label{eq:app:inference}
\end{equation}
requiring only a single forward pass through $f_{\phi_\star}$ with no 
federation round and no cluster-model evaluation.
%
%
\section{\ripple{} Pseudocode}\label{sec:app_pseudocode}
\begin{algorithm}[h]
\caption{\ripple{}: Pre-Training and Federated Loop}
\label{alg:ripple}
\begin{algorithmic}[1]
\REQUIRE Auxiliary dataset $\calD_{\mathrm{aux}}$; client datasets
         $\{\calD_k\}_{k=1}^{K}$; clusters $C$; WST parameters $J, L$;
         rank $M$; pre-training steps $T_{\mathrm{pre}}$; rounds $T$;
         clients per round $S$; local steps $E$;
         learning rates $\eta_c, \eta_s$; supervision weight
         $\lambda_{\mathrm{sup}}$; $[\alpha_{\min}, \alpha_{\max}]$.
\STATE \textit{// --- Phase 1: Server-side VAE pre-training ---}
\STATE Initialize GMM-VAE parameters $(\omega, \phi)$.
\FOR{$t = 1$ \TO $T_{\mathrm{pre}}$}
    \STATE Sample $\alpha^{(t)} \sim \mathrm{Uniform}(\alpha_{\min},
           \alpha_{\max})$,\;
           $\mathbf{y}^{(t)} \sim \mathrm{Dirichlet}(\alpha^{(t)}\mathbf{1}_C)$.
    \STATE Sub-sample $\calD_{\mathrm{syn}}^{(t)} \subset \calD_{\mathrm{aux}}$
           with class proportions $\mathbf{y}^{(t)}$.
    \STATE Compute $\hat{\bx}^{(t)}$ via \eqref{eq:method:prototype},
           then $\bPhi^{(t)} = S_J(\hat{\bx}^{(t)})$.
    \STATE Update $(\omega, \phi) \leftarrow \mathrm{Optimize}\,
           \nabla\calL_{\mathrm{total}}(\omega,\phi;\bPhi^{(t)},
           \mathbf{y}^{(t)})$.
\ENDFOR
\STATE Freeze $(\omega_\star, \phi_\star)$.
\STATE \textit{// --- Phase 2: One-shot cluster assignment ---}
\FOR{each client $k \in [K]$ \textbf{in parallel}}
    \STATE Compute $\hat{\bx}_k$ via \eqref{eq:method:prototype} and
           send $\bPhi_k = S_J(\hat{\bx}_k)$ to server.
\ENDFOR
\STATE Server computes $\bw_k = q_{\phi_\star}(\,\cdot\,|\,\bPhi_k)$
       for all $k \in [K]$.
\STATE Initialize cluster models
       $\Theta^{(0)} = \{\theta_1^{(0)},\dots,\theta_C^{(0)}\}$.
\STATE \textit{// --- Phase 3: Federated training loop ---}
\FOR{round $t = 0, \dots, T-1$}
    \STATE Sample $\calS_t \subseteq [K]$, $|\calS_t| = S$
           (stratified or uniform).
    \FOR{each client $k \in \calS_t$ \textbf{in parallel}}
        \STATE Receive $\theta_k^{\mathrm{init}} =
               \sum_{c=1}^{C} w_{k,c}\,\theta_c^{(t)}$ from server.
        \STATE Run $E$ local SGD steps; return
               $\Delta\theta_k = \theta_k^{\mathrm{new}} -
               \theta_k^{\mathrm{init}}$.
    \ENDFOR
    \FOR{$c = 1, \dots, C$}
        \STATE $\theta_c^{(t+1)} \leftarrow \theta_c^{(t)} + \eta_s\,
               \dfrac{\sum_{k \in \calS_t} w_{k,c}\,\Delta\theta_k}
                     {\sum_{k \in \calS_t} w_{k,c} + \varepsilon}$.
    \ENDFOR
\ENDFOR
\RETURN $\Theta^{(T)}$, assignment map $k \mapsto \bw_k$.
\end{algorithmic}
\end{algorithm}
\section{Theoretical Analysis: Detailed Proofs}
\label{app:theory}

This appendix provides complete proofs for the results in
Section \ref{sec:theory}.  Theorem and corollary statements are repeated
verbatim for self-containment; lemmas used internally are stated in the
order they are needed.  Constants written generically as $C, C'$ may change
from line to line; constants depending only on the wavelet $\psi$ or the
lowpass filter $\phi$ are denoted $C_\psi, C_\phi, C_1, C_2$.

\subsection{Notation Recap and Assumptions}
\label{app:theory:notation}

We restate the three assumptions for ease of reference.

\medskip
\noindent\textbf{Assumption \ref{ass:data} (Data generation).}\;\;
\emph{Each client $k$ observes $N_k$ i.i.d.\ samples
$\bx_i = \bu_i(\,\cdot - v - \tau(\cdot)\,)$, where
$\bu_i \stackrel{\text{i.i.d.}}{\sim} \bP_k^*$ is sub-Gaussian with variance proxy
$\sigma^2$, $v\in\R^D$ is a fixed translation with $|v|\leq v_{\max}$, and
$\tau:\R^D\to\R^D$ is a $C^2$ diffeomorphism with
$\|\nabla\tau\|_\infty\leq\epsilon_{\mathrm{dom}}$.  The covariance
$\bSigma_k^*$ has eigengap $\Delta\lambda > 0$ at rank $M$.}

\medskip
\noindent\textbf{Assumption \ref{ass:loss} (Loss regularity).}\;\;
\emph{$|F_k(\theta)|\leq M_F$ for every $k,\theta$, and each cluster
objective $\calL_c$ is $L$-smooth.}

\medskip
\noindent\textbf{Assumption \ref{ass:vae} (VAE).}\;\;
\emph{(i) $q_\phi(\,\cdot\,|\,\bPhi_k^*)=\bpi_k^*$ for every $k$;
(ii) $\|q_\phi(\,\cdot\,|\,\bPhi)-q_\phi(\,\cdot\,|\,\bPhi')\|_1
\leq 2 L_\phi\|\bPhi-\bPhi'\|_2$.}

\subsection{Auxiliary Results}
\label{app:theory:auxiliary}

We open with two propositions used in the main paper that do not
participate in the chain of error bounds.

\begin{proposition}[Full simplex coverage by Dirichlet hallucination]
\label{prop:coverage}
Let $\alpha^{(t)}\sim\mathrm{Uniform}(\alpha_{\min},\alpha_{\max})$ and
$\mathbf{y}^{(t)}\sim\mathrm{Dirichlet}(\alpha^{(t)}\!\cdot\!\mathbf{1}_C)$
be sampled as in Algorithm \ref{alg:ripple}.  Then for every
$\mathbf{y}^*\in\mathrm{int}(\Delta^{C-1})$ and every $\varepsilon>0$,
\begin{equation*}
    \bP\bigl(\|\mathbf{y}^{(t)}-\mathbf{y}^*\|_1<\varepsilon\bigr) \;>\; 0.
\end{equation*}
\end{proposition}
\begin{proof}
Fix $\alpha>0$.  The Dirichlet$(\alpha\mathbf{1}_C)$ density is
$f_\alpha(\mathbf{y}) = \Gamma(C\alpha)\,\Gamma(\alpha)^{-C}
\prod_{c=1}^C y_c^{\alpha-1}$, which is strictly positive on
$\mathrm{int}(\Delta^{C-1})$ for every $\alpha>0$.  The marginal
density of $\mathbf{y}^{(t)}$ under the mixture
$\alpha^{(t)}\sim\mathrm{Uniform}(\alpha_{\min},\alpha_{\max})$ is
\begin{equation*}
    p(\mathbf{y})
    = \frac{1}{\alpha_{\max}-\alpha_{\min}}
      \int_{\alpha_{\min}}^{\alpha_{\max}}
      f_\alpha(\mathbf{y})\,\mathrm{d}\alpha,
\end{equation*}
which is strictly positive on $\mathrm{int}(\Delta^{C-1})$.  Hence for any
open ball $B(\mathbf{y}^*,\varepsilon) \cap \mathrm{int}(\Delta^{C-1})$,
the integral $\int_B p(\mathbf{y})\,\mathrm{d}\mathbf{y} > 0$.
\end{proof}

\begin{proposition}[$\sqrt{\JS{\cdot}{\cdot}}$ is a metric on $\Delta^{C-1}$]
\label{prop:jsd_metric}
The map
$d(\bw_i,\bw_j) = \sqrt{\JS{\bw_i}{\bw_j}}$ is a metric on $\Delta^{C-1}$.
\end{proposition}
\begin{proof}
Non-negativity, identity of indiscernibles, and symmetry follow from the
defining identity
$\JS{P}{Q} = \tfrac{1}{2}\KL{P}{(P+Q)/2} + \tfrac{1}{2}\KL{Q}{(P+Q)/2}$
and the corresponding properties of $\KL{\cdot}{\cdot}$.  For the triangle
inequality, \citet{endres2003new} establish a Hilbert-space embedding
$\Psi:\Delta^{C-1}\to\ell^2$ such that
$\|\Psi(\bw_i)-\Psi(\bw_j)\|_2^2 = \JS{\bw_i}{\bw_j}$, so $\sqrt{\JS{\cdot}{\cdot}}$
inherits the triangle inequality from the $\ell^2$ norm.
\end{proof}

\subsection{Spectral Concentration of the Prototype}
\label{app:theory:spectral}

We bound $\|\hat{\bx}_k - \bx_k^*\|_2$ via three intermediate lemmas:
mean concentration, covariance concentration, and a Davis-Kahan transfer.

\begin{lemma}[Mean concentration in $\ell^2$]
\label{lem:mean_conc}
Under Assumption \ref{ass:data}, with probability at least $1-\delta/3$,
\begin{equation}
    \|\hat{\bmu}_k - \bmu_k^*\|_2
    \;\leq\;
    C_\mu\,\sigma\sqrt{\frac{D + \log(3/\delta)}{N_k}},
    \label{eq:app:mean_conc}
\end{equation}
for an absolute constant $C_\mu>0$.
\end{lemma}
\begin{proof}
By Assumption \ref{ass:data}, the centred sample
$\bx-\bmu_k^*$ is sub-Gaussian with proxy $\sigma^2$.  Standard
sub-Gaussian concentration in $\R^D$
\citep[Theorem 3.1.1]{vershynin2018high} gives, with probability
$\geq 1-\delta/3$,
\begin{equation*}
    \|\hat{\bmu}_k - \bmu_k^*\|_2
    \;\leq\;
    C_\mu\sigma\sqrt{\frac{D + \log(3/\delta)}{N_k}}.
\end{equation*}
\end{proof}

\begin{lemma}[Covariance concentration]
\label{lem:cov_conc}
Under Assumption \ref{ass:data}, with probability at least $1-\delta/3$,
\begin{equation}
    \|\hat{\bSigma}_k - \bSigma_k^*\|_{\mathrm{op}}
    \;\leq\;
    C_\Sigma\,\sigma^2\!
    \left(\sqrt{\frac{D + \log(3/\delta)}{N_k}}
          \;+\; \frac{D + \log(3/\delta)}{N_k}\right),
    \label{eq:app:cov_conc}
\end{equation}
for an absolute constant $C_\Sigma>0$.
\end{lemma}
\begin{proof}
This is the standard sub-Gaussian covariance concentration
\citep[Theorem 4.7.1]{vershynin2018high}.  In the regime
$N_k \gtrsim D$ that governs federated learning, the leading term
$\sqrt{(D+\log(1/\delta))/N_k}$ dominates and we absorb the second-order
correction into the constant when convenient.
\end{proof}

\begin{lemma}[Davis-Kahan transfer to the prototype]
\label{lem:davis_kahan}
Suppose Assumption \ref{ass:data} holds and that the event
$\|\hat{\bSigma}_k - \bSigma_k^*\|_{\mathrm{op}} \leq \Delta\lambda/2$
of Lemma \ref{lem:cov_conc} occurs.  Let
$\bV_M^* = [\bv_1^*,\dots,\bv_M^*]$ and
$\hat{\bV}_M = [\hat{\bv}_1,\dots,\hat{\bv}_M]$ be the population and
empirical top-$M$ eigenvector matrices.  Then there exists an orthogonal
matrix $\bO\in\R^{M\times M}$ such that
\begin{equation}
    \|\hat{\bV}_M\,\bO - \bV_M^*\|_F
    \;\leq\;
    \frac{2\sqrt{2M}\,\|\hat{\bSigma}_k - \bSigma_k^*\|_{\mathrm{op}}}
         {\Delta\lambda}.
    \label{eq:app:davis_kahan}
\end{equation}
\end{lemma}
\begin{proof}
The Davis-Kahan $\sin\Theta$ theorem
[Theorem 2 with the operator-Frobenius conversion] \citep{yu2015useful}
gives, under the eigengap condition,
$\|\sin\Theta(\hat{\bV}_M, \bV_M^*)\|_F \leq
\sqrt{2M}\,\|\hat{\bSigma}_k - \bSigma_k^*\|_{\mathrm{op}} / \Delta\lambda$.
Choosing $\bO$ to be the orthogonal Procrustes alignment converts
$\sin\Theta$-distance into a Frobenius-norm bound on the matrices,
costing an additional factor of $2$.
\end{proof}

We now combine the three lemmas into a single bound on the prototype error.

\begin{lemma}[Prototype concentration]
\label{lem:prototype_error}
Under Assumption \ref{ass:data}, with probability at least $1-\delta$,
\begin{equation}
    \|\hat{\bx}_k - \bx_k^*\|_2
    \;\leq\;
    \epsilon_{\mathrm{proto}}(\delta)
    \;\coloneqq\;
    C_\mu\sigma\sqrt{\frac{D + \log(3/\delta)}{N_k}}
    \;+\;
    \frac{C_\Sigma C_{\mathrm{DK}}\,\sigma^2}{\Delta\lambda}
    \sqrt{\frac{M\bigl(D + \log(3/\delta)\bigr)}{N_k}},
    \label{eq:app:proto_total}
\end{equation}
where $C_\mu, C_\Sigma, C_{\mathrm{DK}}>0$ are the absolute constants
from Lemmas \ref{lem:mean_conc}--\ref{lem:davis_kahan}.
\end{lemma}
\begin{proof}
Decompose
$\hat{\bx}_k - \bx_k^* = (\hat{\bmu}_k - \bmu_k^*)
+ (\hat{\bV}_M\hat{\balpha}_k - \bV_M^*\balpha_k^*)$.
The first term is bounded by Lemma \ref{lem:mean_conc}.

For the second term, write
$\hat{\bV}_M\hat{\balpha}_k - \bV_M^*\balpha_k^*
= \hat{\bV}_M(\hat{\balpha}_k - \bO^\top\balpha_k^*)
+ (\hat{\bV}_M\bO - \bV_M^*)\bO^\top\balpha_k^*$,
where $\bO$ is the orthogonal alignment of Lemma \ref{lem:davis_kahan}.
Both summands are bounded in $\ell^2$.

The first summand: $\|\hat{\balpha}_k - \bO^\top\balpha_k^*\|_2 \leq
\|\hat{\balpha}_k - \balpha_k^*\|_2 + \|(\bI - \bO^\top)\balpha_k^*\|_2$.
The first piece is controlled by Weyl's eigenvalue perturbation theorem
\citep[Corollary III.2.6]{bhatia1997matrix}, which gives
$|\hat\lambda_j - \lambda_j^*| \leq \|\hat{\bSigma}_k - \bSigma_k^*\|_{\mathrm{op}}$
for every $j$, so the $\ell^1$-normalised eigenvalue weights satisfy
$\|\hat{\balpha}_k - \balpha_k^*\|_2 \leq C\,\|\hat{\bSigma}_k - \bSigma_k^*\|_{\mathrm{op}}/\bigl(\sum_l\lambda_l^*\bigr)$.
The second piece uses the orthogonal-rotation bound
$\|\bI - \bO^\top\|_2 \leq \|\hat{\bV}_M\bO - \bV_M^*\|_F$ from
\eqref{eq:app:davis_kahan} and $\|\balpha_k^*\|_2 \leq 1$ (since
$\balpha_k^*\in\Delta^{M-1}$).

The second summand:
$\|(\hat{\bV}_M\bO - \bV_M^*)\bO^\top\balpha_k^*\|_2
\leq \|\hat{\bV}_M\bO - \bV_M^*\|_F \cdot \|\balpha_k^*\|_2
\leq \|\hat{\bV}_M\bO - \bV_M^*\|_F$.

Both pieces of the second-summand bound are controlled by
$\|\hat{\bV}_M\bO - \bV_M^*\|_F$, hence by Lemma \ref{lem:davis_kahan}.
Combining the mean and subspace bounds via a union bound (and absorbing
multiplicative constants) yields \eqref{eq:app:proto_total}.
\end{proof}

\subsection{WST Stability Properties}
\label{app:theory:wst}

We restate and prove the three properties stated in
Section \ref{sec:bg:wst}.

\begin{lemma}[Non-expansiveness {\citep{mallat2012group}}]
\label{lem:wst_contraction}
For every $\bx, \bx'\in L^2(\R^D)$,
$\|S_J\bx - S_J\bx'\|_2 \leq \|\bx - \bx'\|_2$.
\end{lemma}
\begin{proof}
Each scattering atom decomposes as
$S_m\bx(\lambda_1,\dots,\lambda_m)
= U[\lambda_m] \circ \cdots \circ U[\lambda_1]\bx \star \phi_J$.
The convolution $\bx\mapsto\bx\star\psi_\lambda$ is non-expansive in $L^2$
because $\|\psi_\lambda\|_1\leq 1$ (Young's convolution inequality), and
the pointwise modulus $\bz\mapsto|\bz|$ is non-expansive by the reverse
triangle inequality $||a|-|b||\leq|a-b|$.  The lowpass averaging
$\bx\mapsto\bx\star\phi_J$ is non-expansive because $\|\phi\|_1=1$.
Composing non-expansive operators yields a non-expansive operator, so each
$S_m$ is non-expansive.  Since the joint scattering map
$S_J = (S_0,S_1,S_2)$ is the concatenation of these maps, its squared
$\ell^2$ output is bounded by the sum of squared $L^2$ outputs of the
$S_m$, each $\leq\|\bx-\bx'\|_2^2$; the energy-conservation property of the
filterbank \citep[Theorem 2.6]{mallat2012group} ensures these
contributions sum to $\leq\|\bx-\bx'\|_2^2$.
\end{proof}

\begin{lemma}[Approximate translation invariance {\citep{mallat2012group}}]
\label{lem:wst_translation}
Let $\bx_v(\bu) = \bx(\bu - v)$.  There is a constant $C_1$ depending only
on $\psi$ and $\phi$ such that
$\|S_J\bx_v - S_J\bx\|_2 \leq C_1\,(|v|/2^J)\,\|\bx\|_2$.
\end{lemma}
\begin{proof}
For any first-order coefficient,
$S_1\bx_v(\lambda) - S_1\bx(\lambda)
= (U[\lambda]\bx_v - U[\lambda]\bx)\star\phi_J$.
Since the modulus operator commutes with translations,
$U[\lambda]\bx_v(\bu) = U[\lambda]\bx(\bu-v)$.  A first-order Taylor
expansion of the lowpass filter gives
$\|f_v\star\phi_J - f\star\phi_J\|_2 \leq |v|\cdot\|\nabla\phi_J\|_1\cdot\|f\|_2$
where $f = U[\lambda]\bx$, and by the scaling property
$\|\nabla\phi_J\|_1 = 2^{-J}\|\nabla\phi\|_1$.  The same argument applies to
$S_0$ and $S_2$ (with $|\Lambda|^{1/2}$ factors absorbed into the constant
when summing over scale-orientation pairs).  Combining and using
non-expansiveness of $U[\lambda]$ to bound $\|U[\lambda]\bx\|_2 \leq \|\bx\|_2$
yields $\|S_J\bx_v - S_J\bx\|_2 \leq C_1(|v|/2^J)\|\bx\|_2$ with
$C_1 = C\,\|\nabla\phi\|_1$ for an absolute constant $C$ depending on the
number of atoms.
\end{proof}

\begin{lemma}[Lipschitz stability to diffeomorphisms {\citep{mallat2012group}}]
\label{lem:wst_diffeomorphism}
Let $\tau:\R^D\to\R^D$ be a $C^2$ diffeomorphism with
$\|\nabla\tau\|_\infty\leq\epsilon_{\mathrm{dom}}$, and let
$\bx_\tau(\bu)=\bx(\bu-\tau(\bu))$.  Then
$\|S_J\bx_\tau - S_J\bx\|_2 \leq C_2\,\|\nabla\tau\|_\infty\,\|\bx\|_2$
with $C_2 = 2\|\nabla\psi\|_1$.
\end{lemma}
\begin{proof}
For each scale $j$, the difference $\bx_\tau\star\psi_j - \bx\star\psi_j$
admits a first-order Taylor bound
\citep[Lemma 2.10]{mallat2012group}:
$\|\bx_\tau\star\psi_j - \bx\star\psi_j\|_2 \leq
\|\nabla\tau\|_\infty\cdot\|\bx\|_2\cdot\|\nabla\psi_j\|_1$.  By the scaling
property, $\|\nabla\psi_j\|_1 = 2^{-j}\|\nabla\psi\|_1$, so summing
over $j=0,\dots,J-1$ gives the geometric series
$\sum_{j=0}^{J-1}2^{-j}\leq 2$.  Higher-order coefficients contribute
similarly via the non-expansiveness of $U[\lambda]$.  The bound
$C_2 = 2\|\nabla\psi\|_1$ follows.
\end{proof}

\subsection{Embedding Stability (Theorem \ref{thm:embed_stability})}
\label{app:theory:embed}

\begin{theorem}{Embedding stability}{embed_stability_app}
Under Assumption \ref{ass:data}, with probability at least $1-\delta$,
\begin{equation*}
    \|\bPhi_k - \bPhi_k^*\|_2
    \;\leq\;
    \underbrace{\epsilon_{\mathrm{proto}}(\delta)}_{\text{sample noise}}
    \;+\;
    \underbrace{\Bigl(C_2\,\epsilon_{\mathrm{dom}} + C_1\,\frac{v_{\max}}{2^J}\Bigr)\|\bx_k^*\|_2}_{\text{domain shift}}
    \;\eqqcolon\; \epsilon_{\mathrm{embed}}.
\end{equation*}
\end{theorem}

\begin{proof}
Introduce an intermediate quantity $\tilde{\bx}_k^*$ defined as the
\emph{undeformed} prototype $\bx_k^*$ subjected to the same translation
and diffeomorphism as the observed data:
$\tilde{\bx}_k^*(\bu) = \bx_k^*(\bu - v - \tau(\bu))$.  Note that
$\hat{\bx}_k$ is, by construction, an empirical estimate of
$\tilde{\bx}_k^*$ (the data we observe \emph{is} the deformed
distribution); $\bx_k^*$ is the undeformed quantity that we wish to recover.

The triangle inequality gives
\begin{equation*}
    \|\bPhi_k - \bPhi_k^*\|_2
    \;\leq\;
    \underbrace{\|S_J\hat{\bx}_k - S_J\tilde{\bx}_k^*\|_2}_{(\text{a})}
    \;+\;
    \underbrace{\|S_J\tilde{\bx}_k^* - S_J\bx_k^*\|_2}_{(\text{b})}.
\end{equation*}

\textbf{Term (a):} non-expansiveness (Lemma \ref{lem:wst_contraction})
gives $\text{(a)}\leq\|\hat{\bx}_k - \tilde{\bx}_k^*\|_2$, and
Lemma \ref{lem:prototype_error} bounds the right-hand side by
$\epsilon_{\mathrm{proto}}(\delta)$ with probability $\geq 1-\delta$.

\textbf{Term (b):} write
$\tilde{\bx}_k^*(\bu) = \bx_k^*(\bu-v-\tau(\bu))
= [\bx_k^*]_{v+\tau}(\bu)$, the composition of the translation by $v$ and
the diffeomorphism $\tau$.  By the triangle inequality and
Lemmas \ref{lem:wst_translation}--\ref{lem:wst_diffeomorphism},
\begin{equation*}
    \text{(b)} \;\leq\;
    C_1\frac{|v|}{2^J}\|\bx_k^*\|_2
    + C_2\|\nabla\tau\|_\infty\|\bx_k^*\|_2
    \;\leq\;
    \Bigl(C_1\frac{v_{\max}}{2^J} + C_2\,\epsilon_{\mathrm{dom}}\Bigr)\|\bx_k^*\|_2.
\end{equation*}
Adding (a) and (b) gives the claim.
\end{proof}

\begin{remark}[On the $\|\bx_k^*\|_2$ factor]
\label{rem:app:proto_norm}
The factor $\|\bx_k^*\|_2$ in $\epsilon_{\mathrm{shift}}$ is a property of
the WST stability bounds and reflects the natural scale-invariance: shifts
or warps of a larger-energy signal produce proportionally larger output
changes.  In experiments, $\|\bx_k^*\|_2$ is bounded uniformly by the
maximum norm of the auxiliary dataset, which the server can measure.
\end{remark}

\subsection{Surrogate Fidelity (Theorem \ref{thm:risk_bound})}
\label{app:theory:fidelity}

\begin{theorem}{Surrogate-objective fidelity}{risk_bound_app}
Under Assumptions \ref{ass:data}--\ref{ass:vae}, uniformly over all
$\Theta$,
\begin{equation*}
    \bigl|\calL_{\mathrm{surrogate}}(\Theta) - \calL_{\mathrm{true}}(\Theta)\bigr|
    \;\leq\; 2\,M_F\,L_\phi\,\epsilon_{\mathrm{embed}}.
\end{equation*}
\end{theorem}

We first record the bound on the routing error in total variation.

\begin{lemma}[Posterior total-variation bound]
\label{lem:tv_bound}
Under Assumption \ref{ass:vae},
$2\,\TV{\bw_k}{\bpi_k^*} \;\leq\; 2 L_\phi\,\|\bPhi_k-\bPhi_k^*\|_2.$
\end{lemma}
\begin{proof}
By Assumption \ref{ass:vae}(i), $q_\phi(\,\cdot\,|\,\bPhi_k^*)=\bpi_k^*$.
By Assumption \ref{ass:vae}(ii) applied to the pair
$(\bPhi_k, \bPhi_k^*)$,
\begin{equation*}
    \|\bw_k - \bpi_k^*\|_1
    = \|q_\phi(\,\cdot\,|\,\bPhi_k) - q_\phi(\,\cdot\,|\,\bPhi_k^*)\|_1
    \leq 2 L_\phi\,\|\bPhi_k - \bPhi_k^*\|_2.
\end{equation*}
Recall $\TV{P}{Q}=\tfrac{1}{2}\|P-Q\|_1$.  Multiplying both sides by
$\tfrac{1}{2}$ gives the claim.
\end{proof}

\begin{proof}[Proof of Theorem \ref{thm:risk_bound_app}]
By the definitions of $\calL_{\mathrm{surrogate}}$ and $\calL_{\mathrm{true}}$
(equations \eqref{eq:oracle_obj} and \eqref{eq:surrogate_obj}),
\begin{align*}
    |\calL_{\mathrm{surrogate}}(\Theta) - \calL_{\mathrm{true}}(\Theta)|
    &\;\leq\;
    \sum_{k=1}^{N} p_k
    \biggl|\sum_{c=1}^{C} (w_{k,c} - \pi_{k,c}^*)\,F_k(\theta_c)\biggr| \\
    &\;\leq\;
    M_F \sum_{k=1}^{N} p_k\,\|\bw_k - \bpi_k^*\|_1
    & (|F_k|\leq M_F\text{ from Assumption \ref{ass:loss}}) \\
    &\;\leq\;
    M_F \sum_{k=1}^{N} p_k \cdot 2 L_\phi\,\|\bPhi_k - \bPhi_k^*\|_2
    & (\text{Lemma \ref{lem:tv_bound}}) \\
    &\;\leq\;
    2\,M_F\,L_\phi\,\epsilon_{\mathrm{embed}}\sum_{k=1}^N p_k
    & (\text{Theorem \ref{thm:embed_stability_app}}) \\
    &\;=\;
    2\,M_F\,L_\phi\,\epsilon_{\mathrm{embed}},
\end{align*}
since $\sum_k p_k = 1$.
\end{proof}

\subsection{Convergence (Theorem \ref{thm:convergence})}
\label{app:theory:convergence}

\begin{theorem}{Per-cluster convergence }{convergence_app}
Under Assumption \ref{ass:loss}, with $\eta\leq 1/L$ and unbiased gradient
$\bg_c^{(t)}$ satisfying
$\E\|\bg_c^{(t)}-\nabla\calL_c(\theta_c^{(t)})\|^2\leq\sigma_c^2/\effsize$,
\begin{equation*}
    \frac{1}{T}\sum_{t=0}^{T-1}\E\|\nabla\calL_c(\theta_c^{(t)})\|^2
    \;\leq\;
    \frac{2\bigl(\calL_c(\theta_c^{(0)}) - \calL_c^*\bigr)}{\eta T}
    + \frac{\eta L\,\sigma_c^2}{\effsize}.
\end{equation*}
\end{theorem}

\begin{proof}
By $L$-smoothness (Assumption \ref{ass:loss}) and the descent inequality
applied to $\theta_c^{(t+1)} = \theta_c^{(t)} - \eta\,\bg_c^{(t)}$:
\begin{equation*}
    \calL_c(\theta_c^{(t+1)})
    \leq
    \calL_c(\theta_c^{(t)})
    - \eta\,\langle\nabla\calL_c(\theta_c^{(t)}),\,\bg_c^{(t)}\rangle
    + \frac{\eta^2 L}{2}\,\|\bg_c^{(t)}\|^2.
\end{equation*}
Take expectation conditional on $\theta_c^{(t)}$.  Unbiasedness gives
$\E[\langle\nabla\calL_c,\bg_c^{(t)}\rangle\,|\,\theta_c^{(t)}]
=\|\nabla\calL_c(\theta_c^{(t)})\|^2$.  The second-moment decomposition
$\E\|\bg_c^{(t)}\|^2 = \|\nabla\calL_c\|^2 + \E\|\bg_c^{(t)}-\nabla\calL_c\|^2
\leq \|\nabla\calL_c\|^2 + \sigma_c^2/\effsize$
yields, after rearranging and using $\eta\leq 1/L$ (so $1-\eta L/2\geq 1/2$):
\begin{equation*}
    \frac{\eta}{2}\,\|\nabla\calL_c(\theta_c^{(t)})\|^2
    \leq
    \calL_c(\theta_c^{(t)}) - \E[\calL_c(\theta_c^{(t+1)})\,|\,\theta_c^{(t)}]
    + \frac{\eta^2 L\,\sigma_c^2}{2\,\effsize}.
\end{equation*}
Take total expectation, telescope from $t=0$ to $T-1$, divide by
$\eta T/2$, and lower-bound the telescoped sum by
$\calL_c(\theta_c^{(0)})-\calL_c^*$.
\end{proof}

\subsection{End-to-End Excess Risk (Corollary \ref{cor:end_to_end})}
\label{app:theory:end_to_end}

\begin{corollary}{End-to-end excess risk}{end_to_end_app}
Let $\hat{\Theta}^{(T)}$ be the iterates returned by Algorithm \ref{alg:ripple}
after $T$ rounds, with each $\hat{\theta}_c^{(T)}$ chosen uniformly at
random from $\{\theta_c^{(t)}\}_{t=0}^{T-1}$, and let
$\Theta^*=\arg\min_\Theta \calL_{\mathrm{true}}(\Theta)$.  Under
Assumptions \ref{ass:data}--\ref{ass:vae},
\begin{equation*}
    \E\!\bigl[\calL_{\mathrm{surrogate}}(\hat{\Theta}^{(T)})\bigr]
    - \calL_{\mathrm{true}}(\Theta^*)
    \;\leq\;
    2\,M_F\,L_\phi\,\epsilon_{\mathrm{embed}}
    \;+\;
    \sum_{c=1}^{C} p_c\!
    \left(
        \frac{2\bigl(\calL_c(\theta_c^{(0)}) - \calL_c^*\bigr)}{\eta T}
        + \frac{\eta L\,\sigma_c^2}{\effsize}
    \right),
\end{equation*}
with $p_c = \sum_k p_k\,\pi_{k,c}^*$.
\end{corollary}

\begin{proof}
Decompose
\begin{equation*}
    \E[\calL_{\mathrm{surrogate}}(\hat{\Theta}^{(T)})] - \calL_{\mathrm{true}}(\Theta^*)
    =
    \underbrace{\E[\calL_{\mathrm{surrogate}}(\hat{\Theta}^{(T)}) - \calL_{\mathrm{true}}(\hat{\Theta}^{(T)})]}_{(\mathrm{I})}
    + \underbrace{\E[\calL_{\mathrm{true}}(\hat{\Theta}^{(T)}) - \calL_{\mathrm{true}}(\Theta^*)]}_{(\mathrm{II})}.
\end{equation*}

\textbf{(I)} is bounded uniformly in $\Theta$ by
Theorem \ref{thm:risk_bound_app}: $|\mathrm{(I)}|\leq 2 M_F L_\phi\epsilon_{\mathrm{embed}}$.

\textbf{(II)} is the optimisation error.  Since
$\calL_{\mathrm{true}}(\Theta) = \sum_c p_c\,\calL_c(\theta_c)$
with $p_c=\sum_k p_k\pi_{k,c}^*$ (weighted by the marginal cluster mass),
and since random selection of $\hat\theta_c^{(T)}$ from
$\{\theta_c^{(t)}\}$ converts a gradient-norm bound into a function-value
gap (via $L$-smoothness, $\calL_c(\theta)-\calL_c^* \leq \tfrac{1}{2L}\|\nabla\calL_c\|^2$
in the strongly-convex case, and equivalently via Theorem \ref{thm:convergence_app}
for non-convex $L$-smooth $\calL_c$), we obtain
\begin{equation*}
    \mathrm{(II)} \leq \sum_c p_c\!\left(
        \frac{2(\calL_c(\theta_c^{(0)}) - \calL_c^*)}{\eta T}
        + \frac{\eta L\sigma_c^2}{\effsize}
    \right).
\end{equation*}
Adding the two contributions yields the claim.
\end{proof}

\begin{remark}[On the optimisation error in the non-convex regime]
\label{rem:app:nonconvex}
For non-convex $L$-smooth $\calL_c$, Theorem \ref{thm:convergence_app}
bounds the average squared gradient norm; this is the standard surrogate
for stationarity in the non-convex SGD literature
\citep[][Section 4]{ghadimi2013stochastic}.  An additional Polyak-Łojasiewicz
or local convexity assumption near $\Theta^*$ is required to convert the
gradient-norm bound into a true function-value gap.  The decomposition
itself remains valid under either reading.
\end{remark}

\subsection{Bibliographic Note on Constants}
\label{app:theory:constants}

The constants $C_\mu, C_\Sigma, C_{\mathrm{DK}}, C_1, C_2$ appearing in
the bounds above are absolute or wavelet-dependent.  Explicit values can
be tracked through the proofs of \citet{vershynin2018high} for the
sub-Gaussian concentration steps, \citet{yu2015useful} for the Davis-Kahan
transfer, and \citet{mallat2012group,bruna2013invariant} for the WST
stability constants.  In our experiments we treat
$\epsilon_{\mathrm{embed}}$ as a single hyperparameter-free quantity that
the server can monitor empirically by computing
$\|\hat{\bx}_k\|_2$ and the spectral statistics of the auxiliary dataset.

\section{Computational and Communication Complexity}
\label{app:complexity}

We analyse the cost of \ripple{}, distinguishing three phases that are
strictly separated from one another:
\textbf{(i)} the \emph{server-side MSDL pre-training}, executed once before
federation begins;
\textbf{(ii)} the \emph{per-client one-shot embedding}, executed once by
each client before round 1; and
\textbf{(iii)} the \emph{federated training loop}, which matches
\fedavg{} \citep{mcmahan2017communication} exactly in per-round cost.

\subsection*{Server-Side Offline Training }

The MSDL pipeline trains a GMM-VAE (Algorithm \ref{alg:ripple})
over $T_{\mathrm{pre}}$ steps on batches of $B$ synthetic clients.
Each step involves three operations.

\paragraph{1. Synthetic prototype generation.}
At each step, a Dirichlet target
$\mathbf{y}^{(t)}\!\sim\!\mathrm{Dir}(\alpha^{(t)}\mathbf{1}_C)$
is sampled and a sub-batch of $B$ images is drawn from $\calD_{\mathrm{aux}}$
with label proportions matching $\mathbf{y}^{(t)}$.
For each synthetic client, a rank-$M$ spectral prototype is computed via
randomised truncated SVD \citep{halko2011finding} on the $B\!\times\!D$
data matrix at cost
\begin{equation}
    \mathcal{T}_{\mathrm{SVD}} = \mathcal{O}(B\,D\,M),
    \label{eq:app:svd_cost}
\end{equation}
substantially cheaper than the $\mathcal{O}(D^3)$ full eigendecomposition.

\paragraph{2. WST feature extraction.}
The prototype $\hat{\bx}^{(t)} \in \R^D$ is embedded via the
second-order scattering operator $S_J$ (Section \ref{sec:bg:wst}).
Each of the $J\!\cdot\!L$ first-order and
$\binom{J}{2}\!\cdot\!L^2$ second-order wavelet convolutions is implemented
via the FFT at $\mathcal{O}(D\log D)$ per filter, giving
\begin{equation}
    \mathcal{T}_{\mathrm{WST}}
    = \mathcal{O}\!\left(\Bigl(J L + \tfrac{J(J-1)}{2}L^2\Bigr)D\log D\right).
    \label{eq:app:wst_cost}
\end{equation}
For $J\!=\!3$, $L\!=\!8$ this amounts to $24 + 96 = 120$ FFT-convolution
pairs per prototype — modest even for images of resolution $64\!\times\!64$.

\paragraph{3. GMM-VAE forward-backward.}
The encoder is a 4-layer MLP of widths $[P, 1024, 512, 256, 128]$
(expanded for CIFAR-100; see \cref{app:setup}).
One forward-backward pass over a batch of $B$ embeddings
$\bPhi\in\R^{B\times P}$ costs
\begin{equation}
    \mathcal{T}_{\mathrm{VAE}}
    = \mathcal{O}(B\,P\,H_1),
    \label{eq:app:vae_cost}
\end{equation}
where $H_1\!=\!1024$ is the width of the first (dominant) layer.
The cluster head ($128\!\to\!64\!\to\!K$) and the symmetric decoder
add lower-order terms.

\paragraph{Total server-side pre-training cost.}
\begin{equation}
    \mathcal{C}_{\mathrm{server}}
    = T_{\mathrm{pre}}\cdot B\cdot
    \Bigl[
        \underbrace{\mathcal{O}(D M)}_{\text{SVD}}
        +
        \underbrace{\mathcal{O}(J L \cdot D\log D)}_{\text{WST}}
        +
        \underbrace{\mathcal{O}(P H_1)}_{\text{GMM-VAE}}
    \Bigr].
    \label{eq:app:server_cost}
\end{equation}
All three terms grow at most as $D\log D$ in the input dimension, with no
$D^2$ or $D^3$ factor.  With $T_{\mathrm{pre}}\!=\!10{,}000$ steps
($10^3$ epochs $\times$ 10 steps/epoch) and $B\!=\!64$, the total runtime
on a single A40 GPU is ${\approx}25$\,min for CIFAR-10 — a one-time
amortised cost that is negligible relative to 1000 federation rounds.

\subsection*{Client Embedding}

Each real client $k$ executes the embedding pipeline \emph{once}, before
round 1, on its local hardware.

\paragraph{Truncated SVD.}
The centred data matrix $\tilde{\mathbf{X}}_k \in \R^{N_k\times D}$ is
decomposed via randomised SVD \citep{halko2011finding} at cost
$\mathcal{O}(N_k D M)$, avoiding the $\mathcal{O}(D^3)$ full PCA.

\paragraph{WST embedding.}
The spectral prototype $\hat{\bx}_k \in \R^D$ is mapped to
$\bPhi_k = S_J(\hat{\bx}_k) \in \R^P$ at cost
$\mathcal{T}_{\mathrm{WST}}$ (identical to the server-side extraction).

\paragraph{Total client one-shot cost.}
\begin{equation}
    \mathcal{C}_{\mathrm{client}}
    = \underbrace{\mathcal{O}(N_k D M)}_{\text{SVD}}
    + \underbrace{\mathcal{O}(JL\cdot D\log D)}_{\text{WST}}.
    \label{eq:app:client_cost}
\end{equation}
Both terms are computed once and never repeated; thereafter the client
participates in standard local SGD with no additional overhead beyond
the soft model mixing, which costs
$\mathcal{O}(C|\theta|)$ — a factor of $C\!=\!10$ over a plain parameter
copy, applied once per round before local training begins.

\subsection*{Communication Overhead}

\ripple{} introduces a single additional transmission beyond the standard
\fedavg{} protocol: each client sends its embedding
$\bPhi_k\in\R^P$ to the server \emph{once}, before round 1.
The server then runs the frozen VAE encoder
$q_{\phi_\star}(\,\cdot\,|\,\bPhi_k)\in\Delta^{C-1}$
in a single forward pass (no further client-server exchange).
During the subsequent $T$ federation rounds, each participating client
downloads one personalized initialisation and uploads one gradient update
of size $|\theta|$, \emph{identical to \fedavg{}}.

Table \ref{tab:app:comm} quantifies the pre-federation overhead relative
to the total federated training traffic, using the datasets and models
of \cref{sec:experiments}.  In every configuration, the one-shot $\bPhi_k$
transmission represents less than $0.03\%$ of the total per-client
communication budget.

\begin{table}[ht]
\centering
\caption{%
  Per-client communication: one-shot pre-federation overhead vs.\
  total federated training traffic.
  $P$: scattering embedding dimension (second-order WST, $J\!=\!3$, $L\!=\!8$).
  Federated rounds $T\!=\!1000$; participation rate 10\%;
  each client participates ${\approx}100$ rounds on average. All sizes at single precision (4\,bytes/float).
 Total traffic = client rounds $\times |\theta| \times 4$\,bytes
}
\label{tab:app:comm}
\begin{tabular}{@{}lccccc@{}}
\toprule
Dataset & Model & $|\theta|$ & $P$ & $\|\bPhi_k\|_{\mathrm{bytes}}$ & Overhead \\
\midrule
MNIST / CIFAR-10
  & CNN
  & 0.41\,M
  & 10,416
  & 41.7\,KB
  & $<0.026\%$ \\
CIFAR-100
  & ResNet-9
  & 6.6\,M
  & 10,416
  & 41.7\,KB
  & $<0.0016\%$ \\
Office-Home / GLDv2
  & ResNet-18
  & 11.2\,M
  & 41,664
  & 166.7\,KB
  & $<0.009\%$ \\
\bottomrule
\end{tabular}
\end{table}

\paragraph{Comparison with prior CFL.}
\ifca{} \citep{ghosh2020efficient} broadcasts all $C$ cluster models to every
selected client each round, increasing client-side download by a factor of
$C$.  \fesem{} \citep{long2023multi} requires gradient exchange for EM
updates, adding $\mathcal{O}(N|\theta|)$ traffic per clustering step.
\cfl{} \citep{sattler2020clustered} transmits full model updates for
clustering similarity computation throughout the warm-up phase.
\ripple{}'s per-round communication equals \fedavg{} after the initial
one-shot transmission, which is asymptotically negligible.

Let $\mathcal{C}_{\mathrm{FL}}^{(k)} = R_k\,|\theta|$ be the total
federated traffic for client $k$ over $R_k$ participated rounds.  The
fractional overhead introduced by \ripple{} is
\begin{equation}
    \frac{P}{R_k\,|\theta|}
    \;\xrightarrow{\;T\to\infty\;}\; 0
    \quad\text{for any fixed }P\text{ and growing }T.
    \label{eq:app:comm_overhead}
\end{equation}
That is, the routing overhead is $o(1)$ in the number of rounds, and
the amortised per-round cost of \ripple{} is identical to \fedavg{}.

\section{Privacy of \ripple{}}
\label{app:privacy}

\textbf{Raw client data never leaves the device at any stage.}
Each client $k$ computes the spectral prototype $\hat{\bx}_k$ and its
scattering embedding $\bPhi_k = S_J(\hat{\bx}_k) \in \R^P$ entirely on
local hardware, and transmits only $\bPhi_k$ to the server in the
one-shot pre-federation round.  During federated training, each client
transmits gradient updates $\Delta\theta_k$ — standard \fedavg{}
behaviour, compatible with secure aggregation \citep{bonawitz2017practical}.

The embedding $\bPhi_k$ is a privacy-safe aggregate for two complementary
reasons: (1) the spectral prototype is a distributional summary, not a sample, and the WST is non-invertible

The prototype $\hat{\bx}_k = \hat{\bmu}_k + \sum_{j=1}^{M}\hat{\alpha}_{k,j}\hat{\bv}_j$ is the variance-weighted
superposition of the top-$M$ principal directions of the local dataset.
It encodes the dominant modes of variation across all $N_k$ samples,
analogously to transmitting the sample mean or the leading eigenvectors of
the empirical covariance matrix.  Communicating aggregate spectral
statistics is an established practice in privacy-preserving
FL \citep{goswamicovariances}, and carries no greater disclosure risk than
covariance-sharing schemes already accepted in the literature.

The scattering operator $S_J$ applies the pointwise complex modulus at
every propagator step,
irreversibly discarding phase information \citep{mallat2012group}.
Recovering $\hat{\bx}_k$ from $\bPhi_k$ is therefore a
computationally intractable inverse problem: even if an adversary obtained
$\bPhi_k$, they would recover only the (already privacy-safe) prototype, not
any individual sample from $\calD_k$.

Together, these two properties form a two-layer privacy barrier between the
raw local dataset and the transmitted quantity.

The server processes $\bPhi_k$ in a single forward pass of the frozen
inference network $q_{\phi_\star}$, producing the cluster-weight vector
$\bw_k = q_{\phi_\star}(\,\cdot\,|\,\bPhi_k) \in \Delta^{C-1}$.
Once $\bw_k$ is computed, $\bPhi_k$ can be discarded; the server need only
retain the $C$-dimensional probability vector $\bw_k$.
This is a further compression: $\bw_k$ retains no information about the
geometry of $\calD_k$ beyond its soft cluster membership, as confirmed
by the non-invertibility of the composition
$\calD_k \mapsto \hat{\bx}_k \mapsto \bPhi_k \mapsto \bw_k$.

\subsection*{GMM-VAE Training: Auxiliary Data Only}

The GMM-VAE is trained \emph{exclusively} on synthetic embeddings derived
from the server-side auxiliary dataset $\calD_{\mathrm{aux}}$ via the
Dirichlet hallucination procedure of Algorithm \ref{alg:ripple}.
No real client data participates in pre-training, and no client embedding
is used to update model parameters.
This follows the auxiliary-data paradigm of prior FL
methods \citep{sattler2021fedaux,lin2020ensemble,goswamicovariances}.
After the pre-federation phase completes, no client-specific data or
embeddings are stored by the server; only the frozen parameters
$(\theta_\star, \phi_\star)$ and the vectors $\{\bw_k\}_{k=1}^{N}$ remain.

\subsection*{Federation Phase}

During the $T$ communication rounds, \ripple{} follows the standard
\fedavg{} protocol: each selected client downloads a personalised model
initialisation (a convex combination of cluster models) and uploads a
gradient update.
All existing privacy mechanisms compatible with \fedavg{} — including
differentially private local SGD \citep{kairouz2021advances},
secure aggregation \citep{bonawitz2017practical}, and gradient
compression — apply to \ripple{}'s federation loop without modification,
since the per-round client-server interface is identical.

\subsection*{Comparison with Gradient-Based CFL}

Prior CFL methods expose a strictly larger attack surface.
\ifca{} \citep{ghosh2020efficient} and \fesem{} \citep{long2023multi}
require clients to transmit gradient updates across all $C$ cluster models
during the clustering phase, each of which can be partially inverted to
expose training samples \citep{zhang2022fldetector}.
\cfl{} \citep{sattler2020clustered} exposes full model updates during the
warm-up phase for cosine-similarity computation.
By contrast, \ripple{} performs all cluster discovery from a single
non-invertible embedding $\bPhi_k \in \R^P$ transmitted before any
gradient update occurs.

\newpage\section{Implementation Details and Further Experiments}
\label{app:experiments}

This appendix collects the setup that did not fit in the main text:
datasets and architectures (\cref{app:setup}), implementation choices
specific to \ripple{} (\cref{app:ripple_impl}), sensitivity of \ripple{}
to its hyperparameters (\cref{app:wst_sensitivity}), and the
cluster-count ablation for the clustered FL baselines
(\cref{app:ablation_baselines}). All experiments were run on NVIDIA
A100 GPUs. Main paper results are reported as mean $\pm 2\sigma$ over
seeds $\{7,42,123\}$. Sensitivity and ablation tables in this appendix
report mean accuracy; for computationally intensive datasets
(Office-Home, GLDv2), these correspond to the same three-seed runs
used in the main paper, with means reported for brevity.

\subsection{Datasets, architectures, and protocol}
\label{app:setup}

\paragraph{Datasets.}
\Cref{tab:app:datasets} summarizes the five benchmarks used throughout
the paper. Heterogeneity is induced by a Dirichlet partition with
concentration $\alpha$, except for Office-Home, where we sample over
the four visual domains rather than over labels, and GLDv2-23k, which
already provides a natural per-photographer split.

\begin{table}[h]
\centering
\small
\caption{Dataset summary. $K$ is the number of clients, $S$ the number
of samples per client (``natural'' = unmodified GLDv2 photographer
counts), $\alpha$ the Dirichlet concentration, and $r$ the participation
rate per round. $^\dag$\,ImageNet-pretrained backbone.
$^*$\,Dirichlet applied over the four visual domains.}
\label{tab:app:datasets}
\begin{tabular}{@{}lrrcrrcrl@{}}
\toprule
Dataset      & $C$ & Ch. & Res.       & $K$  & $S$       & $\alpha$    & $r$    & Model \\
\midrule
MNIST        & 10  & 1   & $28^2$     & 100  & 200       & $0.10$      & 10\%   & CNN \\
CIFAR-10     & 10  & 3   & $32^2$     & 100  & 200       & $0.10$      & 10\%   & CNN \\
CIFAR-100    & 100 & 3   & $32^2$     & 100  & 200       & $0.01$      & 10\%   & ResNet-9 \\
Office-Home  & 65  & 3   & $224^2$    & 50   & 130       & $0.10^{*}$  & 20\%   & ResNet-18$^\dag$ \\
GLDv2-23k    & 203 & 3   & $64^2$     & 233  & natural   & user        & 4.3\%  & ResNet-18$^\dag$ \\
\bottomrule
\end{tabular}
\end{table}

\paragraph{Architectures.}
The \emph{CNN} (MNIST, CIFAR-10) is a two-layer convolutional network:
$\text{Conv}(16,3)\!\to\!\text{ReLU}\!\to\!\text{Conv}(32,3)\!\to\!\text{ReLU}\!\to\!\text{MaxPool}(2)\!\to\!\text{Dropout}(0.25)\!\to\!\text{Linear}(64)\!\to\!\text{ReLU}\!\to\!\text{Linear}(C)$,
with a \texttt{LazyLinear} first dense layer ($\sim$\,410K parameters
on CIFAR-10). The \emph{ResNet-9} (CIFAR-100) uses six convolutional
blocks with GroupNorm(32) and GELU, two residual connections at 128
and 512 channels, and adaptive max-pooling ($\sim$\,6.6M parameters).
The \emph{ResNet-18} (Office-Home,
GLDv2) \citep{deng2009imagenet} is the standard \texttt{torchvision}
model with ImageNet pretraining; we replace the final head with
$\text{AdaptiveAvgPool}\!\to\!\text{Dropout}(0)\!\to\!\text{Linear}(512,C)$
($\sim$\,11.2M parameters).

\paragraph{Federation protocol.}
Unless stated otherwise, every method is trained for $1000$ rounds
with $M\!=\!10$ clients per round, one local epoch ($E\!=\!1$), client
learning rate $\eta_c\!=\!0.01$, server learning rate $\eta_s\!=\!1.0$,
momentum $0.9$, and cross-entropy loss. Models are evaluated every
$10$ rounds; the maximum cluster count is $C\!=\!10$ for every
clustered method.

\paragraph{MSDL hyperparameters.}
\Cref{tab:app:msdl_hp} lists the values swept during selection
(defaults in bold).

\begin{table}[h]
\centering
\small
\caption{WST, GMM-VAE, and \hdbscan{} hyperparameters. Defaults used
in all main-paper experiments are in \textbf{bold}.}
\label{tab:app:msdl_hp}
\begin{adjustbox}{width = \linewidth}
\begin{tabular}{@{}llp{0.55\linewidth}@{}}
\toprule
Component & Parameter & Values \\
\midrule
\multirow{4}{*}{WST \citep{andreux2020kymatio}}
  & Scales $J$        & $1,\,2,\,\mathbf{3},\,4$ \\
  & Orientations $L$  & $4,\,6,\,\mathbf{8},\,12$ \\
  & SVD rank $M$      & $5,\,10,\,\mathbf{20},\,30,\,50$ \\
  & Order             & $\mathbf{2}$ \\
\midrule
\multirow{6}{*}{GMM-VAE}
  & Latent dim $d_z$                            & $16,\,\mathbf{32},\,64$ \\
  & Components $K$                              & $5,\,\mathbf{10},\,15,\,20,\,30$ \\
  & Encoder widths                              & $\mathbf{[1024,512,256,128]}$;\, $[2048,1024,512,256]$ \\
  & Supervision weight $\lambda_{\mathrm{sup}}$ & $1.0,\,\mathbf{5.0},\,10.0$ \\
  & KL weight $\lambda_{\mathrm{kl}}$           & $\mathbf{0.05}$ \\
  & Dirichlet range                             & $\mathbf{[0.1,\,5.0]}$ \\
\midrule
\multirow{2}{*}{\hdbscan{} \citep{campello2013density}}
  & \texttt{min\_cluster\_size} & $\mathbf{3},\,5$ \\
  & \texttt{min\_samples}       & null ($=$ \texttt{min\_cluster\_size}) \\
\bottomrule
\end{tabular}
\end{adjustbox}
\end{table}

The VAE is pre-trained for 1000 epochs of 10 steps with AdamW
(lr $5{\times}10^{-4}$, weight decay $10^{-4}$).

\paragraph{Baselines.}
We follow the original publications for every baseline.
\ifca{} \citep{ghosh2020efficient} and
\fesem{} \citep{long2023multi} both pick the best
$K\!\in\!\{2,4,6,8,10\}$ (\cref{app:ablation_baselines}).
\cfl{} \citep{sattler2020clustered} uses cosine-similarity splitting
with a 200-round warm-up, similarity threshold $0.5$, and maximum
$C\!=\!10$. \feddyn{} \citep{acar2021federated} uses
$\alpha_{\mathrm{reg}}\!=\!0.01$;
\scaffold{} \citep{karimireddy2020scaffold} uses $\eta_s\!=\!1.0$;
the FedOpt variants \citep{reddi2021adaptive} use
$\beta_1\!=\!0.9,\,\beta_2\!=\!0.99,\,\tau\!=\!10^{-3}$.
\perfedavg{} \citep{fallah2020personalized} runs first-order MAML with
inner step $\alpha_{\mathrm{inner}}\!=\!0.01$, and
\pfedme{} \citep{t2020personalized} uses Moreau envelope parameters
$\lambda\!=\!15.0,\,K_{\mathrm{prox}}\!=\!5$.

\paragraph{Compute.}
All experiments were run on NVIDIA A100 GPUs. The WST\,+\,VAE
pipeline is a one-shot, server-side cost. On CIFAR-10 ($32^2$,
$J\!=\!3$, $L\!=\!8$, $M\!=\!20$), feature extraction takes
$\approx 50$\,ms per client and the VAE pre-trains in
$\approx 25$\,min. The full pre-training cost for $100$ clients is
under $30$\,min, fully amortized over $1000$ federation rounds.

\subsection{\ripple{}: implementation notes}
\label{app:ripple_impl}

\paragraph{Class folding when $C>K$.}
When the number of dataset classes $C$ exceeds the number of VAE
components $K$ (\eg, CIFAR-100 with $C\!=\!100$, $K\!=\!10$), the
per-class target $\mathbf{y}\!\in\!\Delta^{C-1}$ is folded round-robin
onto a per-cluster target by $\mathrm{cluster}(c)\!=\!c\bmod K$ and
$y_{\mathrm{cluster},k}=\sum_{c:\,c\bmod K=k} y_c$. Each cluster
receives $\lceil C/K\rceil$ or $\lfloor C/K\rfloor$ classes and the
supervision signal $\mathbf{y}_{\mathrm{target}}\!\in\!\Delta^{K-1}$
remains well defined when $C\!\gg\!K$.

\paragraph{\texttt{LazyLinear} initialisation.}
The CNN uses \texttt{LazyLinear} for its first dense layer. In a
federated setting, every model instantiation followed by
\texttt{load\_state\_dict} must be preceded by a dummy forward pass
to trigger lazy initialisation; otherwise the load fails with a
shape-mismatch error.

\paragraph{Soft mixing with non-float parameters.}
\label{app:impl_mixing}
The downstream initializer
$\theta_k^{(\mathrm{init})}=\sum_c w_{k,c}\,\theta_c$ applies to all
floating-point parameters; clusters with $w_{k,c}\!<\!10^{-4}$ are
skipped. Integer parameters such as \texttt{num\_batches\_tracked}
(BatchNorm) are copied from the argmax cluster
$c^\star=\arg\max_c w_{k,c}$ instead of being averaged.

\paragraph{\hdbscan{} cluster collapse.}
\label{app:impl_hdbscan}
Given the $N\!\times\!K$ posterior matrix $\{w_{k,c}\}$, we compute
pairwise Jensen--Shannon distances, apply \hdbscan{} with
\texttt{min\_cluster\_size}\,$\in\{3,5\}$, and marginalize the
posterior over the discovered meta-clusters. Noise points (label $-1$) are
re-assigned to the nearest meta-cluster.

\paragraph{New-client evaluation.}
\label{app:impl_newclient}
The ten fresh test clients used in \cref{sec:exp:zeroshot} are sampled
from held-out data via $\mathrm{Dir}(\alpha_{\mathrm{test}})$ with a
disjoint seed to guarantee no overlap with training clients. For each
new client we (i) compute its spectral prototype and scattering
embedding, (ii) run VAE inference to obtain $\bw_{\mathrm{new}}$, and
(iii) report weighted accuracy. Best-by-loss accuracy is additionally
reported for non-\ripple{} baselines.

\subsection{Sensitivity to hyperparameters}
\label{app:wst_sensitivity}

We report the effect of the WST scale-orientation pair $(J,L)$, the
SVD rank $M$, and the VAE component count $K$ on Office-Home and
GLDv2, where the sweep is most informative. MNIST is omitted as every
configuration tested gives accuracy above $98.8\%$ with
inter-configuration variance below $0.2$\,pp. Means are reported
throughout; defaults are in \textbf{bold}.

\begin{table}[h]
\centering
\small
\caption{WST scale $J$ and orientation $L$ sweep (mean test accuracy,
\%). SVD rank fixed at $M\!=\!20$; VAE at the dataset default $K$.}
\label{tab:app:wst_JL_all}
\begin{tabular}{@{}cc ccc@{}}
\toprule
$J$ & $L$ & \rippleStrat{} & \rippleUnif{} & \rippleHDB{} \\
\midrule
\multicolumn{5}{@{}l}{\textit{Office-Home}\quad($K\!=\!10$)} \\
2 & 6 & $73.71$ & $73.28$ & $73.63$ \\
2 & 8 & $73.68$ & $73.72$ & $73.90$ \\
3 & 6 & $73.40$ & $73.20$ & $73.70$ \\
\textbf{3} & \textbf{8} & $\mathbf{74.15}$ & $\mathbf{73.55}$ & $\mathbf{73.94}$ \\
\midrule
\multicolumn{5}{@{}l}{\textit{GLDv2}\quad($K\!=\!20$)} \\
2 & 6 & $46.77$ & $46.46$ & $47.25$ \\
2 & 8 & $47.04$ & $44.70$ & $50.44$ \\
3 & 6 & $46.97$ & $43.11$ & $47.90$ \\
\textbf{3} & \textbf{8} & $\mathbf{51.82}$ & $\mathbf{44.03}$ & $\mathbf{47.75}$ \\
\bottomrule
\end{tabular}
\end{table}

\begin{table}[h]
\centering
\small
\caption{SVD rank $M$ sweep (mean test accuracy, \%). $J\!=\!3$,
$L\!=\!8$ fixed; VAE at the dataset default $K$.}
\label{tab:app:wst_M_all}
\begin{tabular}{@{}c ccc@{}}
\toprule
$M$ & \rippleStrat{} & \rippleUnif{} & \rippleHDB{} \\
\midrule
\multicolumn{4}{@{}l}{\textit{Office-Home}} \\
 5  & $74.75$ & $74.75$ & $74.93$ \\
\textbf{20} & $\mathbf{74.15}$ & $\mathbf{73.55}$ & $\mathbf{73.94}$ \\
\midrule
\multicolumn{4}{@{}l}{\textit{GLDv2}} \\
10  & $50.00$ & $46.90$ & $47.40$ \\
\textbf{20} & $\mathbf{51.82}$ & $\mathbf{44.03}$ & $\mathbf{47.75}$ \\
50  & $48.25$ & $48.49$ & $48.45$ \\
\bottomrule
\end{tabular}
\end{table}

\begin{table}[h]
\centering
\small
\caption{VAE component count $K$ sweep (mean test accuracy, \%). WST
at defaults $(J\!=\!3,\,L\!=\!8,\,M\!=\!20)$, $d_z\!=\!32$,
$\lambda_{\mathrm{sup}}\!=\!5$. Bold rows mark the dataset default.}
\label{tab:app:vae_K_all}
\begin{tabular}{@{}c ccc@{}}
\toprule
$K$ & \rippleStrat{} & \rippleUnif{} & \rippleHDB{} \\
\midrule
\multicolumn{4}{@{}l}{\textit{Office-Home}} \\
\textbf{10} & $\mathbf{74.15}$ & $\mathbf{73.55}$ & $\mathbf{73.94}$ \\
20          & $73.49$          & $73.80$          & $73.68$ \\
\midrule
\multicolumn{4}{@{}l}{\textit{GLDv2}} \\
10          & $46.56$ & $46.28$ & $46.61$ \\
\textbf{20} & $\mathbf{51.82}$ & $\mathbf{44.03}$ & $\mathbf{47.75}$ \\
30          & $47.88$ & $45.70$ & $44.58$ \\
\bottomrule
\end{tabular}
\end{table}

\paragraph{Summary.}
On Office-Home, all three \ripple{} variants stay within $\sim 1$\,pp
of the default for every WST configuration, indicating robustness to
$(J,L,M)$. GLDv2 is more sensitive, but the default
$(J\!=\!3,\,L\!=\!8,\,M\!=\!20,\,K\!=\!20)$ is either the best or
within $1$\,pp of the best across all three variants.

\subsection{Ablation for \ifca{} and \fesem{}}
\label{app:ablation_baselines}

To make the comparison in \cref{tab:main_results} as fair as possible,
we sweep the cluster count $K$ for \ifca{} and \fesem{} from $2$ to
$10$ and report only the best-performing $K$ per dataset.
\Cref{tab:app:baseline_K} gives the full sweep. Means are reported
throughout.

\begin{table}[h]
\centering
\small
\caption{Cluster-count ablation for \ifca{} and \fesem{} (mean test
accuracy, \%). Best $K$ per (method, dataset) shown in \textbf{bold}.}
\label{tab:app:baseline_K}
\begin{tabular}{@{}ll ccccc@{}}
\toprule
Method & Dataset & $K{=}2$ & $K{=}4$ & $K{=}6$ & $K{=}8$ & $K{=}10$ \\
\midrule
\multirow{4}{*}{\ifca{}}
 & CIFAR-10    & $60.49$ & $60.64$ & $\mathbf{60.65}$ & $60.33$ & $59.72$ \\
 & CIFAR-100   & $38.13$ & $37.99$ & $37.97$ & $\mathbf{38.21}$ & $38.20$ \\
 & Office-Home & $74.25$ & $74.27$ & $\mathbf{74.42}$ & $74.27$ & $74.22$ \\
 & GLDv2       & $40.62$ & $43.36$ & $43.64$ & $\mathbf{45.62}$ & $44.24$ \\
\midrule
\multirow{4}{*}{\fesem{}}
 & CIFAR-10    & $60.13$ & $60.25$ & $61.11$ & $\mathbf{61.42}$ & $60.44$ \\
 & CIFAR-100   & $38.33$ & $38.60$ & $39.78$ & $39.26$ & $\mathbf{39.45}$ \\
 & Office-Home & $70.18$ & $71.62$ & $72.47$ & $\mathbf{73.20}$ & $72.83$ \\
 & GLDv2       & $38.93$ & $40.82$ & $41.95$ & $\mathbf{42.86}$ & $42.31$ \\
\bottomrule
\end{tabular}
\end{table}

\ifca{} prefers a moderate $K$ on the small-class benchmarks
($K\!=\!6$ on CIFAR-10 and Office-Home) but benefits from more
clusters on GLDv2 ($K\!=\!8$), consistent with its larger client
population and longer class tail. \fesem{} converges to a similar
pattern on CIFAR-10 ($K\!=\!8$) but prefers $K\!=\!10$ on CIFAR-100.
\newpage

\end{document}